\pdfoutput=1
\documentclass[10pt,conference]{IEEEtran}

\usepackage{hhline}
\usepackage{multirow}
\usepackage{multicol}
\usepackage{array}
\usepackage{graphicx} 
\usepackage[usenames]{color} 
\usepackage{paralist} 
\usepackage{verbatim} 
\usepackage{url}      
\usepackage{wrapfig}  
\usepackage{booktabs}
\usepackage{colortbl}
\usepackage{tabularx}
\usepackage{amsmath}
\usepackage{color}
\usepackage{comment}
\usepackage{cite}
\usepackage[justification=centering,small]{caption}
\usepackage{caption}
\usepackage{subcaption}
\usepackage{balance}
\usepackage{xspace}
\usepackage{bm}
\usepackage{bbm}
\usepackage[shortlabels]{enumitem}
\usepackage{hyperref}
\usepackage{algorithm}
\usepackage{algorithmic}
\usepackage{amssymb}

\def \ie {\emph{i.e.}, }

\newtheorem{theorem}{Theorem}

\newtheorem{proof}{Proof}

\DeclareMathOperator*{\argmin}{arg\,min}
\DeclareMathOperator*{\argmax}{arg\,max}

\title{Scheduling Real-time Deep Learning Services as Imprecise Computations}

\author{
  \IEEEauthorblockN{
    Shuochao Yao\IEEEauthorrefmark{1},
    Yifan Hao\IEEEauthorrefmark{2},
    Yiran Zhao\IEEEauthorrefmark{3},
    Huajie Shao\IEEEauthorrefmark{1},
    Dongxin Liu\IEEEauthorrefmark{1},\\
    Shengzhong Liu\IEEEauthorrefmark{1},
    Tianshi Wang\IEEEauthorrefmark{1},
    Jinyang Li\IEEEauthorrefmark{1},
    Tarek Abdelzaher\IEEEauthorrefmark{1}}
  \IEEEauthorblockA{
    University of Illinois at Urbana-Champaign\IEEEauthorrefmark{1}, VMware\IEEEauthorrefmark{2} , Pinterest\IEEEauthorrefmark{3} 
  }
  \IEEEauthorblockA{
    Email: \{syao9, yifanh5, zhao97, hshao5, dongxin3, sl29, tianshi3, jinyang7, zaher\}@illinois.edu
  }
}

\begin{document}
\maketitle

\begin{abstract}
The paper presents an efficient real-time scheduling algorithm for {\em intelligent real-time edge services\/}, defined as those that perform machine intelligence tasks, such as voice recognition, LIDAR processing, or machine vision, on behalf of local embedded devices that are themselves unable to support extensive computations. The work contributes to a recent direction in real-time computing that develops scheduling algorithms for machine intelligence tasks with anytime predicition. We show that deep neural network workflows can be cast as imprecise computations, each with a mandatory part and (several) optional parts {\em whose execution utility depends on input data\/}. The goal of the real-time scheduler is to maximize average accuracy of deep neural network outputs, while meeting task deadlines, thanks to opportunistic shedding of the least necessary optional parts. The work is motivated by the proliferation of increasingly ubiquitous but resource-constrained embedded devices (for applications ranging from autonomous cars to the {\em Internet of Things\/}) and the desire to develop services that endow them with intelligence. Experiments on recent GPU hardware and a state of the art deep neural network for machine vision illustrate that our scheme can increase the overall accuracy by 10\% $\sim$ 20\%, while incurring (nearly) no deadline misses.
\end{abstract}

\section{Introduction}
\label{sec:introduction}
This paper develops a scheduling algorithm for deep neural network tasks that maximizes result accuracy while meeting task deadlines. The work observes that modern machine intelligence workflows (specifically, deep neural networks) can be formulated as imprecise computations~\cite{liu1994imprecise} with a mandatory part followed by multiple optional ones. The utility of executing the optional components is input-data-dependent. We implement deep learning workflows as imprecise computations that (i) estimate the improvement in result quality attainable from executing the optional parts, and (ii) shed unnecessary parts as needed to meet deadlines. The algorithm significantly improves the ability of the underlying platform to deliver accurate results, while meeting deadline constraints. 
 
This work is motivated by new applications such as autonomous cars~\cite{kim2013parallel}, drones~\cite{song2014towards}, and smart IoT objects~\cite{kopetz2011internet,stankovic2014research} that rely on increasingly sophisticated sensors to maintain awareness of their environment and perform various detection, recognition, and localization tasks. These tasks typically need a GPU to process the largely parallel sensor data. For example, modern cameras might offer around 1 million pixels per frame and LIDARs will generate up to 10 million 3D points per frame. A single embedded platform, such as an autononous car or a drone will typically have multiple such sensors. Endowing individual sensors with GPUs becomes expensive. A better architecture might offload the data from sensors to a local computationally more capable node over a sufficiently fast local communication fabric (e.g., today, a gigabit switch costs less than \$50). We call such a node, the {\em edge server\/}. Instances of such an architecture recently emerged~\cite{li2017fitcnn,cheng2018adaptive,abdelzaher:19} to support a variety of real-time applications~\cite{bechtel2018deeppicar,mikami2018deepcounter,huang2018case}. The recent NVIDIA AGX platform line-up is one example of today's GPU-enabled nodes designed to be the edge server in such an architecture (NVIDIA AGX is specifically marketed as the ``brain'' node supporting autonomous driving~\cite{yang2019re}). These applications motivate exploration of real-time support (e.g., scheduling policies) for AI tasks running on the edge server node.

Recent directions in real-time support for machine intelligence investigated real-time properties of leading AI platforms (such as the NVIDIA TX1, TX2, and Xavier)~\cite{otterness2017evaluation,yang2018avoiding,pujol2019generating}, discussed neural network approximations (via hyper-parameter tuning) to tune computation to platform capabilities~\cite{bateni2018apnet,yao2017deepiot,yao2018fastdeepiot}, and investigated alternative GPU-based implementations of neural networks that offer real-time guarantees~\cite{pujol2019generating}, as well as modes of interactions with GPUs to improve timeliness, throughput, and accuracy of neural network workflows~\cite{yang2019re,pujol2019generating}. Techniques to reduce communication cost in offloading machine intelligence have also been proposed~\cite{ikeda2018reduction,eshratifar2018energy}. 

This work differs from the state of the art in two respects. First, of the current real-time systems literature, to the authors' knowledge, {\em we are the first to consider a deep neural network model that offers mandatory and optional parts\/}. We thus complement existing approaches (that consider per-layer approximations~\cite{bateni2018apnet} and code optimizations~\cite{yang2019re,pujol2019generating}), offering a new degree of customization freedom. Second, we are the first to {\em consider the dependence of the utility of optional parts on the characteristics of input data\/}. Specifically, we propose a novel utility metric, based on {\em probability of output correctness\/} (that we call {\em confidence\/}), and show that this metric can be predicted on (input) instance-by-instance basis, using a recent technique in deep learning~\cite{yao2018rdeepsense}. As a consequence, our work enables the implementation of a near-optimal algorithm for AI confidence maximization subject to schedulability constraints.   

The rest of this paper is organized as follows. Section~\ref{sec:DeepCrimson} introduces the technical details of the scheduling model. We describe system implementation in Section~\ref{sec:implementation}. The evaluation is presented in Section~\ref{sec:Evaluation}. We introduce related work in Section~\ref{sec:Related} and conclude in Section~\ref{sec:Conclusion}.


\section{Neural Networks as Imprecise Computations}
\label{sec:DeepCrimson}
In this section, we review relevant fundamentals of neural networks, present a model of neural network tasks that breaks them into mandatory and optional stages, and introduce a near-optimal approximation algorithm for accuracy-maximizing scheduling of these stages, based on a simple dynamic programming formulation and a fully polynomial time approximation scheme.

\subsection{Neural Networks Background} 
Consider a platform that runs deep neural network tasks to support intelligent real-time applications. Deep neural networks have a layered structure. External inputs, such as images or sound clips, are applied to the first layer. The output of one layer is then input to the next. A layer can be thought of as a collection of computations (called nodes), that can be performed in parallel, each computing a different feature over the layer's input data. The number of such nodes (called the size of the layer) is the dimensionality of the feature space. Note how inputs to a layer are themselves features computed by the previous layer. Thus, as more layers are stacked, each subsequent layer's features are computed in progressively more complex spaces, as its inputs constitute outputs of more layers of prior processing. The last layer, called the output layer, can directly compute a classification result based on its (already very complex) input feature space. The total number of layers is called neural network {\em depth\/}. On a GPU, it is convenient to keep the size of each layer of the deep neural network the same, allowing parallel layer-by-layer execution of nodes on GPU cores. Layer size can be configured such that all GPU cores are utilized.

Now consider, for example, a network that processes images to detect the presence of various objects, such as humans, cars, or buildings.
More complex images need more complex feature computations, and hence more layers to produce a correct classification; a picture of an empty blue sky will need far fewer layers to yield a high quality classification result compared to complex cluttered images. The needed neural network depth is therefore {\em data-dependent\/}.  Hence, the scheduled depth of neural network processing becomes an interesting scheduling parameter (besides end-to-end deadlines). Since task complexity is data dependent, this parameter is not known {\em a priori\/}. 

A related challenge lies in the lack of well-defined output {\em utility metrics\/} and {\em utility functions\/} to serve as a foundation for deciding on the best neural network depth. Indeed, quantification of utility has always been a challenge in most research aiming at optimizing application-perceived quality metrics. In this paper, we propose {\em accuracy of results\/} as the utility measure. We call this metric, {\em confidence\/} (in correctness). Note that, this metric is largely generalizable, regardless of what the results are used for, since it is always desirable to have accurate outputs. It is also trivial to extend this metric to {\em weighted\/} accuracy, for example, if some tasks are more important than others. This property makes accuracy maximization a widely applicable contribution across a variety of machine learning algorithms and applications. This utility function can be computed using solutions proposed in recent literature that estimate probabilistic confidence in output correctness of deep learning systems  (e.g., confidence in correctness of classifier output)~\cite{yao2018rdeepsense, yao2018apdeepsense}. 

In this section, we introduce the task model and a corresponding scheduling algorithm, based on imprecise computations, that maximizes utility across the task set subject to schedulability constraints. 

\subsection{Anytime Neural Network}
\label{sec:nn-model}
We consider a model, where requests for executing machine intelligence tasks (such as identifying and classifying obstacles in an image) arrive 
in an event-driven fashion. 
For example, if requests are triggered by events such as motion {\em detection\/} (e.g., by a sensor in a smart space) or obstacle {\em detection\/} (e.g., due to LIDAR reflections), triggering a more complex object {\em identification\/} task.

Normally, tasks will be dispatched on a GPU. In this paper, we assume that no GPU sharing occurs. Hence, only one task executes on the GPU at a time. This simplifying assumption is consistent with our ability to configure neural network layer size to match the available number of cores, and allows the work to remain independent of proprietary algorithms that might be implemented within the GPU to handle task sharing. Future extensions of this work can relax this assumptions by considering GPU pipelining/sharing~\cite{yang2019re}, GPU clusters~\cite{jeon2018multi}, and GPU heterogeneity~\cite{pujol2019generating}. The execution of multiple consecutive layers constitutes a {\em stage\/}. Individual stages cannot be preempted. A stage cannot begin to execute until its predecessor is complete.
A scheduler maintains a queue, deciding which stage of which task to pass to the GPU next.

At time, $t$, let the set of tasks that have arrived but have not yet completed be denoted by $\mathcal{J}(t)$, where $|\mathcal{J}(t)| =N(t)$ is the number of tasks. For each task, $J_i \in \mathcal{J}(t)$, let the number of stages in the corresponding neural network be $L_i$. We call the $l$-th stage, $J_{il}$. Hence, we can express $\mathcal{J}_i$ as the set $\mathcal{J}_i= \{J_{il}  \mid l = 1, \cdots, L_i \}$, $\mathcal{J}_i^{l'}$ as the set $\{J_{il} \mid l=1, \cdots, l'\}$. Let the absolute deadline of a task $J_i$ be $d_i$; and let the worst-case execution time of a stage, $J_{il}$, be $p_{il}$. In our system, since most of the processing occurs on the GPU, we are able to ignore the part of the processing that occurs on the CPU. In an architecture where data copying is fast, the CPU part tends to be a small constant that can simply be subtracted from the end-to-end deadline and ignored. Thus, our execution time is approximated by the GPU execution time including any time the GPS spends copying data from and to the CPU. We find it useful to also define the cumulative sum of execution times of the first $L$ stages of a task (for any $L \leq L_i$).  Let that sum for a task, $\mathcal{J}_i$, be denoted by $P_i^L = \sum_{l=1}^L p_{il}$. Similarly, let the finish time of the $L$-th stage, $J_{iL}$, be $F_i^L$, and let the reward achieved for running the first $L$ stages of a task $J_i$ by the deadline (i.e., with $F_i^L \leq d_i$) be denoted by $R_i^L$.  No reward is derived from the execution of stages that miss the task deadline. 


\begin{algorithm}        
\caption{Dynamic programming algorithm}   
\label{alg:DeepCrimson}                     
\begin{algorithmic}[1]
\small 
\STATE Inputs: request index to start, $k$
\FOR {$i \in [ k-1, \cdots, N]$}
\FOR{$r \in [1, \cdots, \lfloor NR\rfloor_{\Delta}]$}
\IF{$r > \lfloor iR\rfloor_{\Delta}$}
\STATE $S(i+1, r)$ = $S(i, r)$; $P(i+1, r)$ = $P(i, r)$; \textbf{continue}
\ENDIF
\STATE Update $S(i+1, r)$ and $P(i+1, r)$ with \eqref{eqn:dynmaic_program0} and \eqref{eqn:dynmaic_program}.
\ENDFOR
\ENDFOR
 \end{algorithmic}
\end{algorithm}
\setlength{\textfloatsep}{0pt}

The purpose of the scheduler is to decide on the number of stages (called {\em depth\/}, $l_i$, to execute for every current task, $\mathcal{J}_i$, such that the total reward, $\sum_{i=1}^N R_i^{l_i}$ is maximized, subject to the schedulability constraint $F_i^{l_i} \leq d_i$. If dropping entire tasks is disallowed, we further constrain the problem to state that at least $ l_i\ge \omega_i$ stages be executed of each task, $\mathcal{J}_i$. We then call stages, $\{ J_{il}: 1 \leq l \leq \omega_i\}$ the {\em mandatory\/} part of the task, and call stages $\{J_{il}: \omega_i < l \leq l_i\}$, {\em optional\/}.
Unfortunately, scheduling with deadlines and rewards is NP-complete. Therefore, we propose a fully polynomial time approximation scheme, taking EDF as the underlying scheduling policy.

Note that, as argued above, the scheduled resource is the GPU, not the CPU. By EDF, we therefore mean that tasks (or, more accurately, individual stages of tasks) are sent from the CPU to the GPU in EDF manner. When preemption is required, it can be performed on stage boundaries. In principle, EDF is suboptimal when non-preemptive regions are involved (because a stage that executes on the GPU cannot be preempted). We find, however, that execution of individual stages in practice takes roughly comparable times. Thus, the worst-case amount of non-preemption per task (which is the execution time of one stage) is roughly the same for all tasks. This amount can simply be subtracted from the end-to-end task deadline upfront and preemptive schedulability conditions used on the resulting deadlines. Since the subtracted amount is roughly the same for all tasks, the order of deadlines remains the same, and EDF remains (approximately) optimal.  

In the rest of this paper, when we mention task deadlines, we shall implicitly refer to deadlines computed after the adjustments above have been performed. That is to say, deadlines that result after (i) the estimated amount of CPU processing and (ii) the amount of GPU non-preemption have been subtracted from the original value.
  
\subsection{Near Optimal Depth Assignment and Stage Scheduling}~\label{sec:DeepCrimson}
Let us quantize the reward function, such that all rewards $\{R_i^L\}$ are measured in multiple of some basic increment $\Delta$. For simplicity, we denote $\lfloor R_i^L/\Delta \rfloor$ as $\lfloor R_i^L \rfloor_\Delta$. This quantization introduces errors but reduces the problem size of our stage scheduling algorithm. In this subsection, we will introduce and analyze such tradeoff between efficiency and optimality.


Without loss of generality, assume that tasks, at time $t$, are indexed according to their absolute deadlines, $\{d_i\}$, such that $d_1 \leq d_2 \leq ... \leq d_N$. Choosing the right depth, $l_i$, for each task (and hence the right reward, $R_i^{l_i}$) can be cast as a dynamic programming problem. 
Let $R$ be the largest reward we can achieve from a single task, \ie $R = \max\{R_i^{l_i}\}$. We can thus upper bound the total reward of $N$ tasks with $N\cdot R$. 
The dynamic programming formulation maintains the two-dimensional table of subproblems, where one dimension represents the number of considered tasks, \ie $i\in\{1,\cdots, N\}$, 
and the other represents quantized total rewards, \ie $r\in\{\Delta, 2\Delta, \cdots, \Delta \lfloor NR\rfloor_{\Delta}\}$.
Each cell in the table is the solution to a subproblem, $S(i,  r)$, that represents {\em optimally\/} selecting the depth for the top $i$ tasks (who have the top-$i$ earliest absolute deadlines, according to our indexing), that attains exactly $r$ cumulative reward in total and takes up the least amount of execution time possible.
 For the notational simplicity, we let $P(i, r)$ be the corresponding least amount of execution time used by the top $m$ tasks that achieve exactly $r$ cumulative reward, with a value of $\infty$ to denote no such solution.

For the first row ($i=1$) are simply the utility function of the task, $\mathcal{J}_1$. If given $P_1^l \le d_1$, we will initialize $S(1, R_1^l)$ and $P(1, R_1^l)$ to $\mathcal{J}_1^l$ and $P_1^l$ respectively with all other cells in $S(1, r)$ and $P(1, r)$ being $ \varnothing$ and $\infty$ respectively.
Since task indexes are sorted in increasing order of deadlines, under EDF, a task with a higher index will execute after tasks with lower indexes. Hence, for subsequent rows, the following recursive relation applies for columns 1 through $\lfloor iR\rfloor_\Delta$, where $\Delta\cdot \lfloor iR\rfloor_\Delta$ is the largest possible quantized total cumulative reward for executing first $i$ tasks. For simplicity, we denote $r - \Delta\cdot \left \lfloor R_i^l\right\rfloor_{\Delta}$ as $\bar{r}_{i}^{l}$,
\vspace{-0.2cm}
\begin{equation}
S(i+1, r) = 
 \begin{cases}
 \displaystyle
\argmin_{l \in \{\omega_{i+1},\cdots, L_{i+1}\}}\left\{ P_{i+1}^{l} + P\left(i,\bar{r}_{i+1}^{l}\right), P(i, r)  \right\}, \\ \phantom{AAAAAAAAA}\text{if } P_{i+1}^{l} + P\left(i, \bar{r}_{i+1}^{l}\right)  \leq d_{i+1}\\
    \varnothing,              \\ \phantom{AAAAAAAAAAAAAAAAAAAa}\text{otherwise}
 \end{cases}  
\label{eqn:dynmaic_program0}
\end{equation}
\begin{equation}
P(i+1, r) = 
 \begin{cases}
 \displaystyle
\min_{l \in \{\omega_{i+1},\cdots, L_{i+1}\}}\left\{ P_{i+1}^{l} + P\left(i, \bar{r}_{i+1}^{l}\right), P(i, r)  \right\}, \\ \phantom{AAAAAAAAA}\text{if } P_{i+1}^{l} + P\left(i, \bar{r}_{i+1}^{l}\right)  \leq d_{i+1}\\
    \infty,               \\ \phantom{AAAAAAAAAAAAAAAAAAAa} \text{otherwise}
 \end{cases}  
\label{eqn:dynmaic_program}
\end{equation}

The dynamic programming algorithm using the recursive relationship in Equations \eqref{eqn:dynmaic_program0} and \eqref{eqn:dynmaic_program} is illustrated in Algorithm~\ref{alg:DeepCrimson}, where $k$ is the starting task index for the current dynamic table update. Namely, when a task $\mathcal{J}_k$ arrives whose deadline is $d_k$, existing table rows for tasks with deadlines $d < d_k$ stay the same. Table rows for tasks with deadlines $d \geq d_k$ (including the new arrival) need to be (re)computed. 

Sorting task indexes by deadline, given an updated table, the optimal solution starts with the cell $S(N,r_{f})$, where $N$ is the task with the largest deadline that arrived and have not yet finished, and $r_{f}$ is the quantized time corresponding to the absolute deadline of that task. 
The value $l_N^*$ in that cell is the optimal depth for task $\mathcal{J}_N$. The recursive step is as follows: After determining the optimal depth, $l_i^*$, for task $i$, from cell $S(i,r)$, we visit cell $S\left(i-1,r - \Delta \cdot \left\lfloor R_i^{l_i^*}\right\rfloor_\Delta \right)$, and decrement $i$ (until we reach $\mathcal{J}_1$).

Intuitively speaking, Algorithm~\ref{alg:DeepCrimson} can achieve the near optimal solution as the reward quantization step $\Delta$ approaches 0. By defining $\Delta = \epsilon R/ N$  and scaling the quantization step with respect to $\epsilon$, we will be able to get a solution that is at least $(1-\epsilon)$  of the optimal one.
\begin{theorem}:
With $\Delta = \epsilon R/ N$, Algorithm~\ref{alg:DeepCrimson} is a $(1-\epsilon)$ approximation of the optimal task stage scheduling.
~\label{thm:approx_bound}
\end{theorem}
\begin{proof}:
Let OPT be the solution to our task stage scheduling problem. Since we have a quantized version with the solution $\mathcal{S}$ proposed by Algorithm~\ref{alg:DeepCrimson}, we can achieve a trivial claim that:
\begin{equation}
\sum_{\mathcal{J}_i^l \in \mathcal{S}} \Delta\cdot \left \lfloor R_i^l \right\rfloor_\Delta \ge \sum_{\mathcal{J}_i^l \in \text{OPT}} \Delta \cdot \left \lfloor R_i^l \right \rfloor_\Delta
\label{eqn:proof_org}
\end{equation}
From the left-hand side of~\eqref{eqn:proof_org}, it is obvious to see that
\begin{equation}
\sum_{\mathcal{J}_i^l \in \mathcal{S}} R_i^l  \ge \sum_{\mathcal{J}_i^l \in \mathcal{S}} \Delta\cdot \left \lfloor R_i^l \right \rfloor_\Delta
\label{eqn:proof_lhs}
\end{equation}
Recall that $\Delta = \epsilon R/N$, and we let the total cumulative reward of optimal scheduling policy be $R_{OPT}$. From the right-hand side of~\eqref{eqn:proof_org}, 
\begin{equation}
\begin{split}
 \sum_{\mathcal{J}_i^l \in \text{OPT}} \Delta \cdot \left \lfloor R_i^l \right \rfloor_\Delta & \ge \Delta  \sum_{\mathcal{J}_i^l \in \text{OPT}} \frac{R_i^l}{\Delta} - 1\\
 & \ge \sum_{\mathcal{J}_i^l \in \text{OPT}} R_i^l  - N \cdot \Delta \\
 & \ge \sum_{\mathcal{J}_i^l \in \text{OPT}} R_i^l  - \epsilon R \\
  & \ge (1-\epsilon)  R_{OPT}
\end{split}
\label{eqn:proof_rhs}
\end{equation}
Combining~\eqref{eqn:proof_org}~\eqref{eqn:proof_lhs}, and~\eqref{eqn:proof_rhs}, we can arrive at the following conclusion.
\begin{equation}
\sum_{\mathcal{J}_i^l \in \mathcal{S}} R_i^l  \ge (1-\epsilon)  R_{OPT}
\end{equation}
\end{proof}

\subsection{Predicting Utility of Future Stages}
\label{sec:stage-utility}
It remains to comment on how utility is computed.
Recent work~\cite{abdelzaher:19} proposed a technique for computing confidence in correctness of deep learning outputs. It groups neural network layers into a small number of stages (of multiple layers each). At the end of each stage, a thin softmax function layer is attached to compute a classification at selected internal layers. The output of the classifier is a vector of class probabilities, where the largest probability is called the classification confidence; the probability that the class selected by the classifier is correct.
Accurate confidence estimation was discussed in several recent publications~\cite{gal2016dropout,yao2018rdeepsense,lakshminarayanan2016simple, yao2018deep}. 
We choose the model described in~\cite{abdelzaher:19} (in its Section II) for utility estimation. The model estimates, for each task $\mathcal{J}_i$, a confidence in correctness of output of the current stage, say, $L_i$, which is a probabilistic value, $R_i^{L_i}$, between 0 and 1. 

To estimate utility of future stages, several heuristics are empirically compared in this paper. These heuristics are: 

\begin{figure*}[!htb]
\centering
\includegraphics[width=0.8\linewidth]{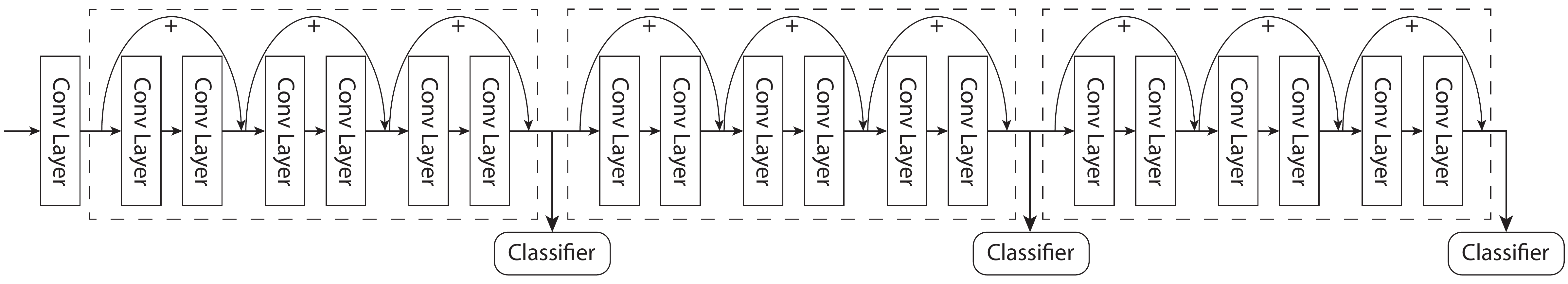}
\caption{An illustration of three-stage ResNet.}
\label{fig:3stag_resnet}
\vspace{-0.3cm}
\end{figure*}

\begin{itemize}
\item
{\em Maximum increase:\/} Assume that the next stage will increase the utility from $R_i^{L_i}$ to 1. Hence, $R_i^{L_i+1} = 1$. This approach favors tasks with the lowest current confidence, since they promise the most increase in utility if executed further.
\item
{\em Exponential increase:\/} Assume that the next stage will reduce distance to 1 in half. Hence, $R_i^{L_i+1} = R_i^{L_i} + 0.5(1 - R_i^{L_i})$. This, as we show in the evaluation later, is the most accurate approximation of utility functions of deep learning systems implemented as imprecise computations.
\item
{\em Linear increase:\/} Assume that the next stage will continue to improve utility linearly. Hence,  $R_i^{L_i+1} = \min (1, R_i^{L_i} \cdot P_i^{L_i+1}/P_i^{L_i})$.
\end{itemize}

\subsection{Updating Depth Assignment}
In principle, the above heuristics can be used to compute the utility of each future stage of a task once its first (mandatory) stage has been executed. In practice, however, it is useful to revisit estimated utility of future stages upon the execution of each new stage of the task. This may result in a problem if the updated estimates change, rendering the original (thought-to-be-optimal) schedule suboptimal. It is easy to show that, if the updated future utility of a stage (of the current task) becomes larger, our previous depth assignment still preserves optimality, since it means we did the right thing running that task. Otherwise, if the updated utility is lower, it might be we are running a suboptimal task and may need to reconsider. In other words, a recalculation of optimal depth assignment is in order. Recalculating the depth assignment with dynamic programming is too cumbersome. Since the current task is always the one with the shortest deadline under EDF, we would have to re-calculate the optimal depth for {\em all\/} subsequent tasks if using dynamic programming. We therefore employ a greedy heuristic instead. The heuristic tries to replace the remaining stages in $\mathcal{J}_1$ (the current task) with the stages in other tasks that can achieve higher cumulative reward. Assume that the remaining stages in $\mathcal{J}_1$ are from depth ${l_1}+1$ to $l_1^*$, and $l_i^*$ is the previous optimal depth selection for task $\mathcal{J}_i$. Thus: 
\begin{equation}
\begin{split}
&\hat{l_i^*} = \argmax_{\substack{i\in\{2, \cdots, N\} \\ l\in\{l_i^*+1, \cdots, L_i\}}} R_i^l - R_i^{l_i^*} \\
&\text{s.t.} \quad \sum_{l = l_i^*}^{l} p_{il} \le \sum_{l'=l_1+1}^{l_1^*} p_{1l'}
\end{split}
\label{eqn:fast_update}
\end{equation}
If $R_i^{\hat{l_i^*}} - R_i^{l_i^*} > R_1^{l_1^*} - R_1^{l_1}$, we replace the depth assignment of $\mathcal{J}_1^{l_1}$ with $\mathcal{J}_i^{\hat{l_i^*}}$. Otherwise, we follow the original depth selection.

\begin{figure}[!htb]
\vspace{-0.3cm}
\centering
\includegraphics[width=0.8\linewidth]{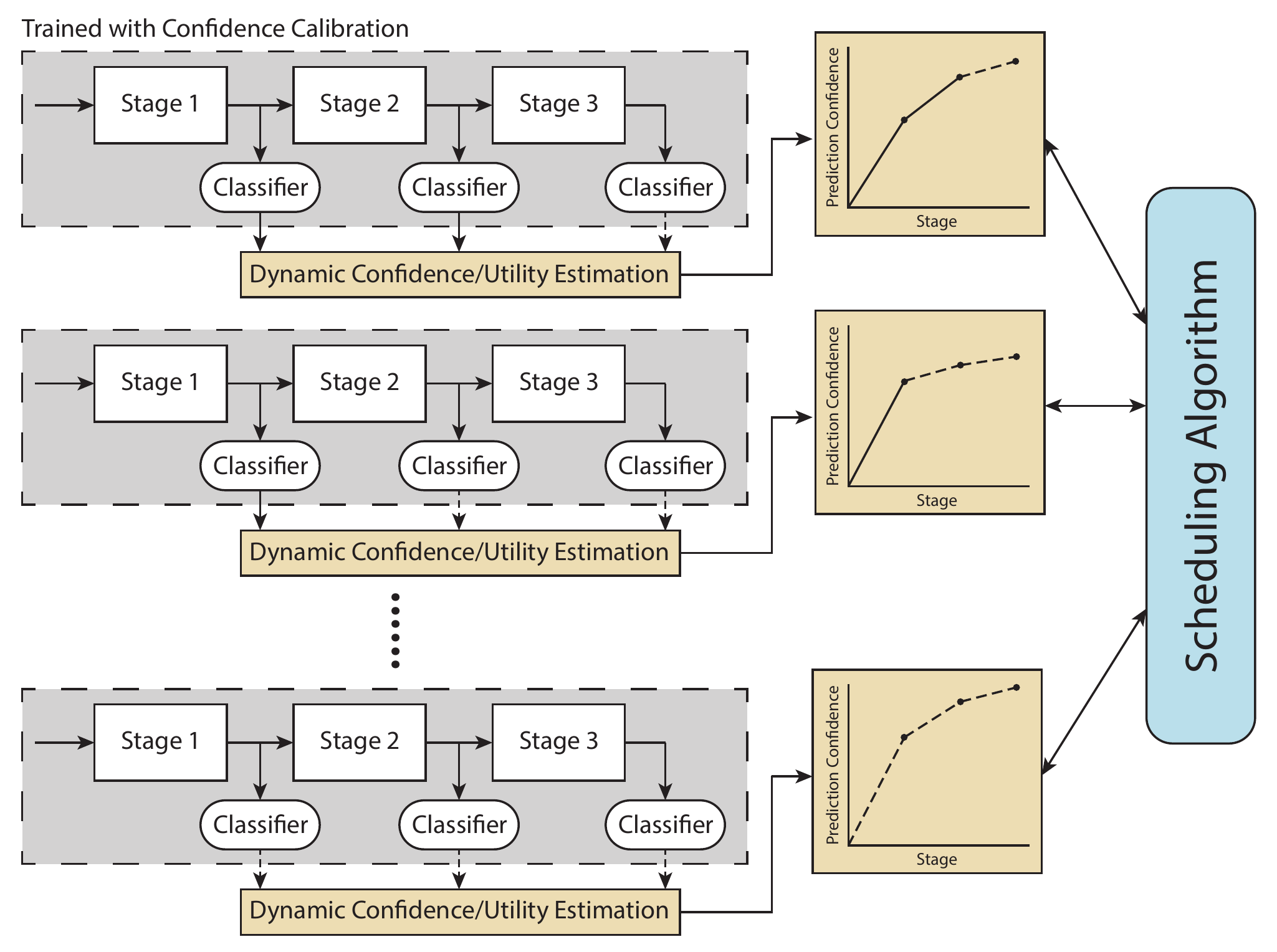}
\caption{System architecture of RTDeepIoT scheduling framework implementation.}
\label{fig:implementation}
\vspace{-0.1cm}
\end{figure}

\section{Implementation}~\label{sec:implementation}
We implemented a user space scheduling framework RTDeepIoT to verify the effectiveness of our imprecise computation model at scheduling neural network tasks. Implementing the scheduler in {\em user space\/} is clearly a disadvantage, as it increases scheduling overhead. We opted for this design decision to simplify experimentation. In essence, the resulting performance becomes a worst-case bound. A more efficient implementation (beyond the scope of this paper) can only make performance better. Having said so, one should also note that a user space implementation has its own advantages. First, it makes the work compatible with many popular operating systems, since no changes are done to the OS scheduler. Second, it facilitates the exploitation of widely deployed deep learning libraries and applications. Specifically, we integrate our scheduler with TensorFlow~\cite{abadi2016tensorflow}. TensorFlow libraries interface internally with GPU APIs, offering the application developer only high-level control of what to process. Our scheduler sits on top of that exported interface, leaving the underlying machine learning libraries in tact. Finally, since utility of future stages may need to be recomputed (upon each stage execution), and since this re-computation is based on returned application-level (confidence) results, the system is already executing in user space when these results become available. The user space is thus a natural place for making the decision on what to run next, as otherwise a system call would be needed anyway to pass the results to a kernel-level scheduler, thus reducing savings. An assessment of performance-savings with a kernel-level scheduler is delegated to future work. 


\subsection{Application Workload \& Benchmark}

The reemergence of deep learning leads to the impressive achievements in computer vision. Vision-based object recognition, labelling objects in a picture, becomes one of the most canonical real-world applications of neural networks. Object recognition also serves as an important building block in many intelligent system, including autonomous vehicles. We therefore choose object recognition as the application workload in this paper. 


In order to verify the efficacy of our neural network stage scheduling framework, we build a proof-of-concept prototypes of object classification service based on the state-of-the-art residual neural networks (ResNet)~\cite{he2016deep}, which are trained and evaluated on CIFAR-10 and ImageNet datasets respectively. Such stage-wise design can be easily applicable to other neural network architectures for computer vision applications~\cite{hu2018anytime}. 

An alternative might have been to evaluate our framework with data sets from the domain of autonomous vehicles or human activity recognition. We opted to keep the scope of the evaluation limited because the used neural networks need to be retrained (to offer mandatory and optional components) for every new application domain we evaluate. Standard pre-trained neural networks generate final outputs only. Instead, in our model, we must train the network to generate both the intermediate results after each stage, and the confidence estimates in these intermendiate results. We thus caveat our evaluation by the fact that it was done on vision-based (object classification) tasks only. While we do not have a reason to believe that other data sets would yield substantially different trends, we do not explore other applications and nueral network types in this paper.

In the rest of the evaluation, we divide the number of layers in ResNet uniformly into three stages. 
As shown in Figure~\ref{fig:3stag_resnet}, a stage contain multiple layers. At the end of each stage, a simple softmax classifier is appended, using the end-of-stage aggregated features for classification. When the execution of a stage finished, it will output a tuple in the form \textit{(predicted value, confidence)}. \textit{Predicted value} is the classification result from the current stage, specifying the most likely classification. \textit{Confidence} describes the likelihood that this classification is correct. For example, a picture can be classified as a cat, dog, or cow, with probabilities 0.6, 0.3, and 0.1, respectively. The classification result is then (``cat", 0.6).

\subsection{System Architecture}
The system architecture of RTDeepIoT scheduling framework is illustrated in Figure~\ref{fig:implementation}. 
The objective of RTDeepIoT framework is to provide an abstraction of real-time neural network execution for service requests. 

The scheduler is invoked upon the occurrence of one of two event types. First, when an object detection service issues a new task request, it sends the absolute deadline and input image to the RTDeepIoT framework through a REST API. Once RTDeepIoT receives the request, the schedule is updated to include (at least) the arrived task's mandatory part. The current depth assignment is updated for all tasks according to Algorithm~\ref{alg:DeepCrimson}. Second, when the previous execution of a neural network stage has finished, RTDeepIoT updates the corresponding task utility prediction according to the newly acquired confidence. The depth selection is updated according to the greedy heuristics~\eqref{eqn:fast_update}. 
Finally, once a task finishes all its scheduled stages up to the assigned depth or once its deadline is missed, RTDeepIoT returns the latest available inference result back and removes all related information from the scheduling table. 

\section{Evaluation}
\label{sec:Evaluation}

In this section, we evaluate the proposed scheduling framework on object detection tasks with two real-world benchmark datasets; CIFAR-10 that consists of 10000 test images of 10 classes, and ImageNet that consists of 50000 test images of 1000 classes. Each test image is an object detection service request. The requests arrive in a random but configureable order to evaluate the effectiveness of proposed scheduling system under diverse workload patterns. 
There are $K$ concurrent clients that generate service requests. 
Within a time interval, each request comes with a relative deadline and a random image selected from the shuffled test dataset. The relative deadline is drawn from a uniform distribution described by two parameters: a maximum relative deadline, $D_u$, and a minimum relative deadline, $D_l$. The stage execution times for service requests are known from prior profiling. Specifically, the worst-case execution time of each neural network stage is measured from the server 10,000 times using training data. We calculate the upper bound of a $99\%$ confidence interval and use it as the worst-case execution time in our evaluation. 
Per our imprecise computation model, the scheduling algorithm may decide to run only some subset of the stages by the deadline. We consider a request to have failed to meet its deadline if none of its computation stages are executed before the deadline. The image classification result from the last executed stage of a request (before its deadline) is used as the final inference result.

The edge server that runs the scheduling framework and the object detection service has an Intel i7-4770 CPU, with 32 GB memory and NVIDIA TITAN X Pascal GPU. The evaluation is performed under Ubuntu 16.04 with kernel version 4.13. The residual neural network for object detection is implemented on TensorFlow 1.14.0, achieving the state-of-the-art testing accuracy on two datasets when having no timing constraints. Note that, the GPU chosen in this evaluation is quite computationally advanced by today's standards. However, since IoT applications that require intelligent edge services may take some time to become commonplace, it is likely that the chosen GPU will be more representative of a midrange system by that time. 

In the following subsections, we will evaluate our RTDeepIoT real-time scheduling framework from different perspectives and microbenchmarks, including utility curve prediction, stage scheduling for utility maximization, system hyper-parameter tuning, and system overhead.

\subsection{Utility Curve Prediction}~\label{sec:utility}
One key insight of our proposed imprecise computation model for neural networks is that the utility of optional parts (i.e., the estimated correctness probability of their outputs) depends on the characteristics of input data.  In Section~\ref{sec:stage-utility}, we proposed several heuristics for estimating the utility of future stages. 

In the following experiments, we set the number of concurrent clients $K$ to $20$;  the minimum relative deadline $D_l$ is $0.01$s; the maximum relative deadline $D_u$ is $0.3$s and $0.8$s for CIFAR10 and ImageNet, respectively.
In addition, we define the time quantization step, $\Delta$, to be $0.1$, which will be discussed in detail in Section~\ref{sec:hyperparameter}.

To understand the impact of predicted utility curves on scheduler performance, 
we evaluate our utility-maximizing scheduling using the three simple heuristic utility estimation methods mentioned in Section~\ref{sec:stage-utility}, and compare them with the optimal (but unrealizable) ``oracle'' policy that knows exactly the computed confidence (i.e., utility) of each stage ahead of time. (To implement the oracle, we simply run each test image through all stages ahead of time, record confidence computed at each stage, and give that information to the oracle before the actual experiment starts.) All algorithms are summarized as follows:
\begin{enumerate}
\item RTDeepIoT-Exp: during utility prediction, assumes that the next stage will reduce distance to 1 in half. Hence, $R_i^{l_i+1} = R_i^{l_i} + 0.5\cdot(1 - R_i^{l_i})$.
\item RTDeepIoT-Max: assumes that the next stage will increase the utility from $R_i^{l_i}$ to 1. Hence, $R_i^{l_i+1} = 1$.
\item RTDeepIoT-Lin: assumes that the next stage will continue to improve utility linearly. Hence,  $R_i^{l_i+1} = \min (1, R_i^{l_i} \cdot P_i^{l_i+1}/P_i^{l_i})$.
\item RTDeepIoT-OPT: knows exactly the computed confidence after each stage beforehand.
\end{enumerate}

Evaluation results are illustrated in Figure~\ref{fig:utility_task} -~\ref{fig:utility_lower}. With all these workload patterns, RTDeepIoT-Exp is almost always the best-performing algorithm, other than the optimal RTDeepIoT-OPT. The experiment shows that in order to estimate utility of future stages of deep learning (vision) workloads, we can use the exponential increase function. In the rest of the evaluation section, unless stated otherwise, the exponentially increasing utility function is assumed. When compared with the optimal policy that knows the computed confidence after each stage beforehand, RTDeepIoT-Exp has comparable accuracy, being within $2\%$ from optimal most of the time. 

\begin{figure}[!htb]
\vspace{-0.3cm}
\begin{subfigure}{1.0\linewidth}
  \centering
  \includegraphics[width=0.7\linewidth]{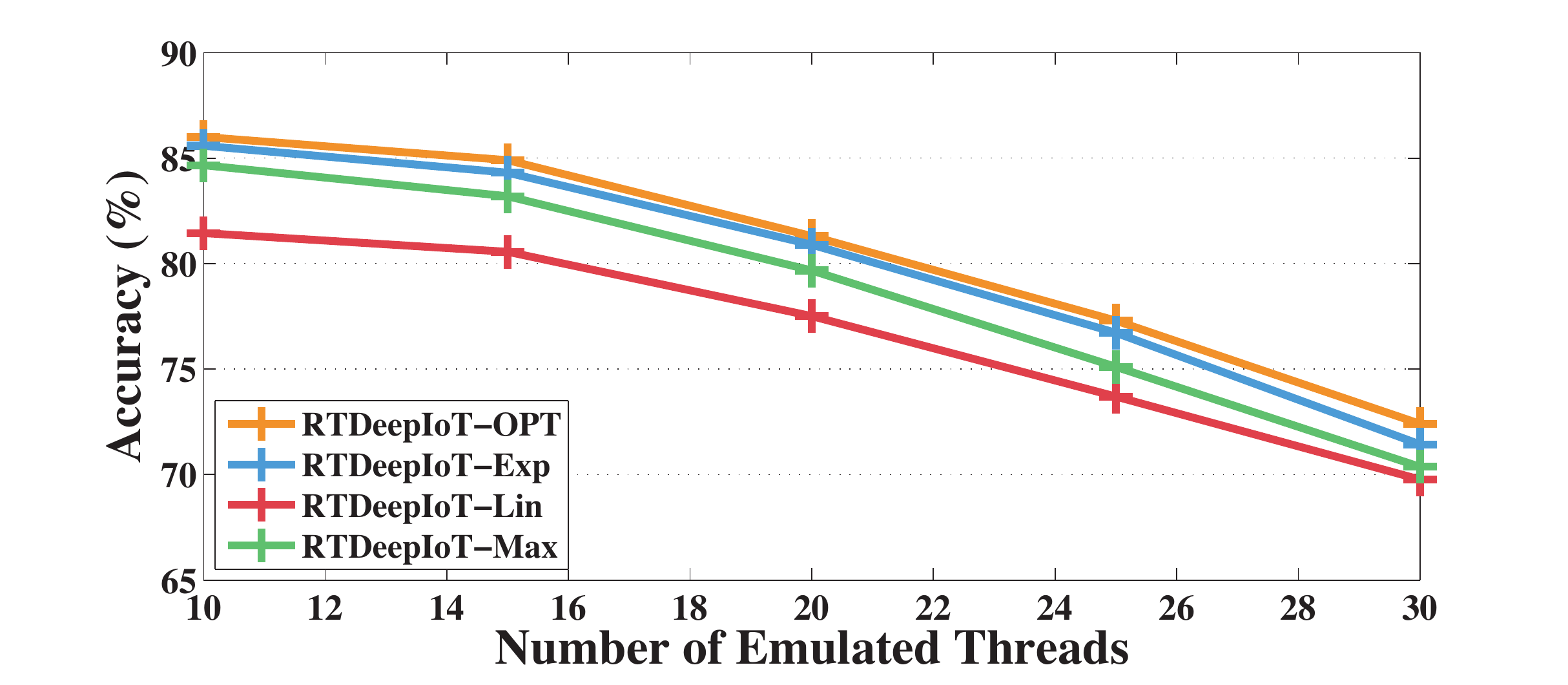}
  \caption{The accuracy with CIFAR10.}
  \label{fig:utility_cifar10_task}
\end{subfigure}
\vspace{-0.2cm}
\begin{subfigure}{1.0\linewidth}
  \centering
  \includegraphics[width=0.7\linewidth]{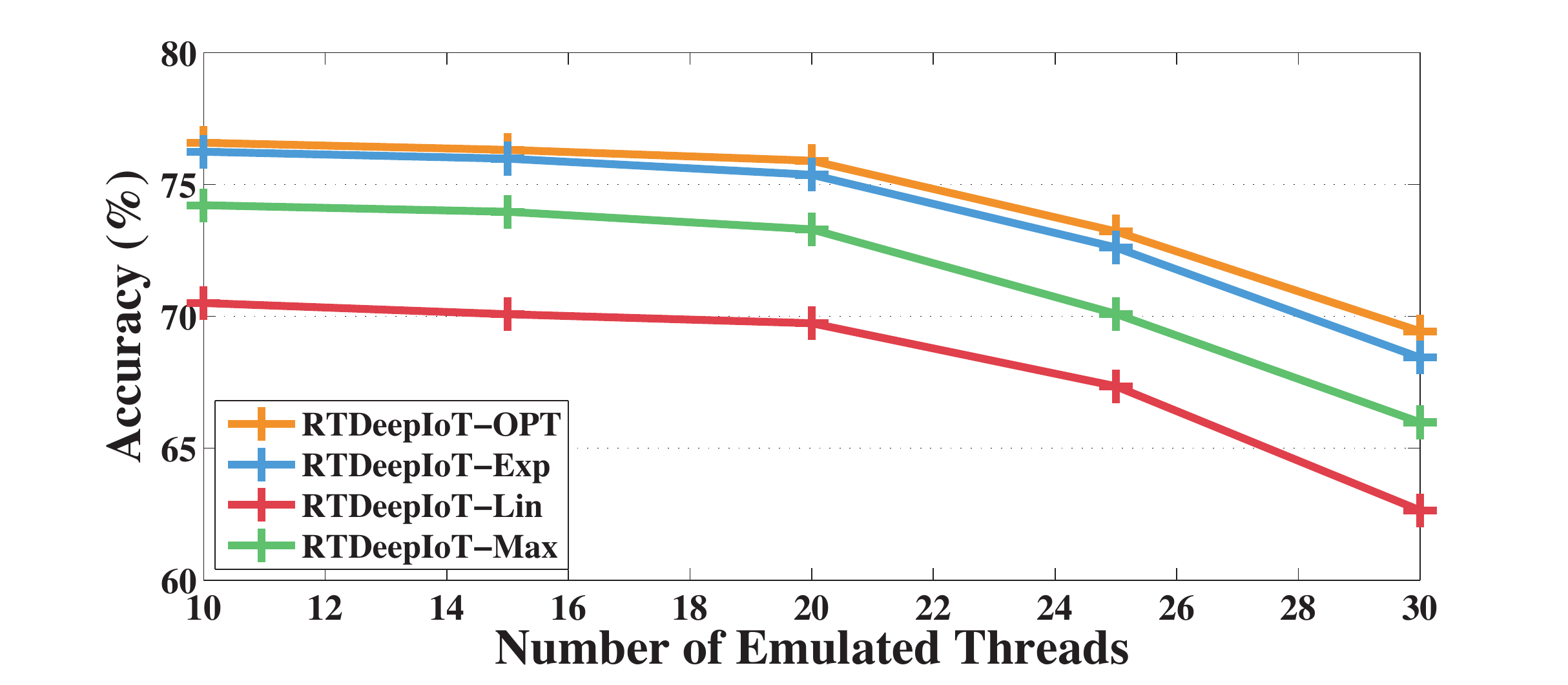}
  \caption{The accuracy with ImageNet.}
  \label{fig:utility_imagenet_task}
  \end{subfigure}
  \caption{Accuracy under $K$ concurrent clients on CIFAR10 and ImageNet.}
  \label{fig:utility_task}
\vspace{-0.3cm}
\end{figure}

\begin{figure}[!htb]
\vspace{-0.3cm}
\begin{subfigure}{1.0\linewidth}
  \centering
  \includegraphics[width=0.7\linewidth]{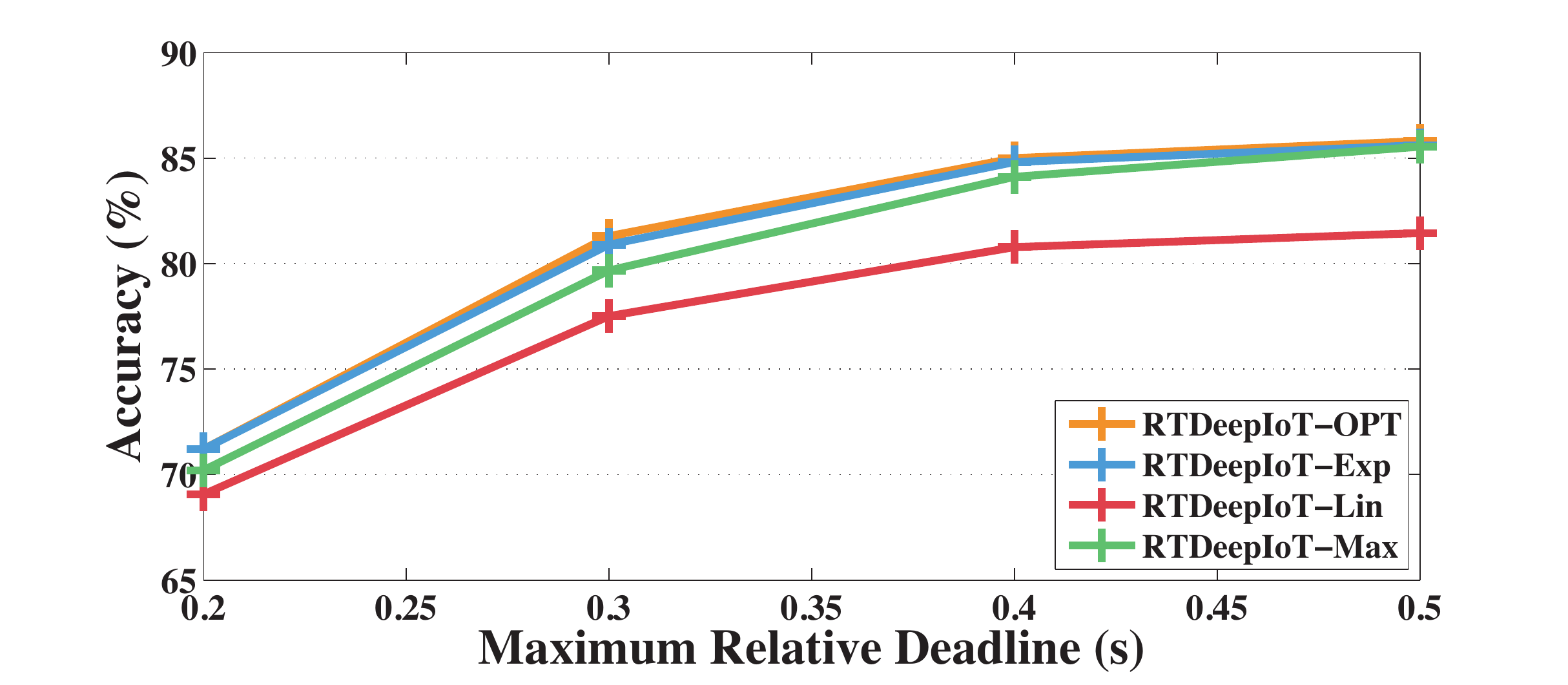}
  \caption{The accuracy with CIFAR10.}
  \label{fig:utility_cifar10_upper}
\end{subfigure}
\vspace{-0.2cm}
\begin{subfigure}{1.0\linewidth}
  \centering
  \includegraphics[width=0.7\linewidth]{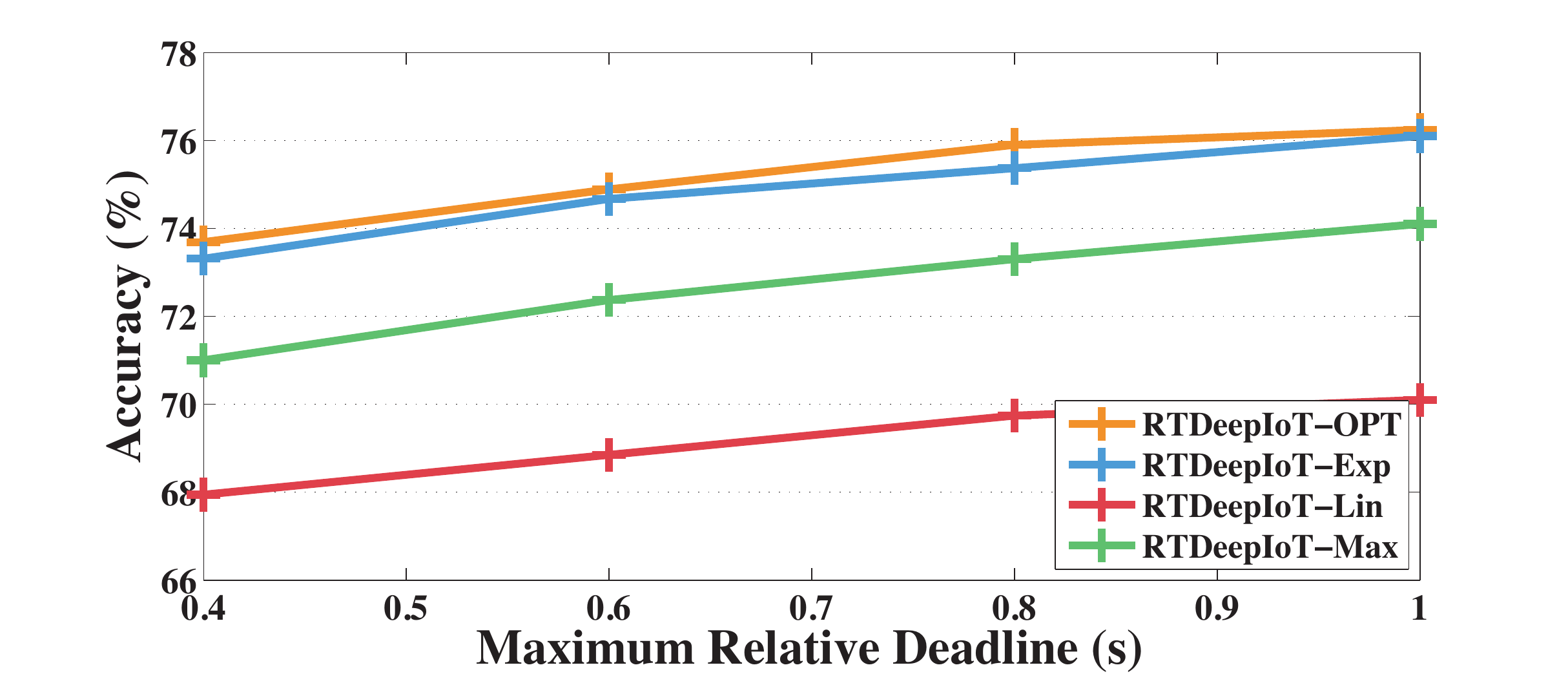}
  \caption{The accuracy with ImageNet.}
  \label{fig:utility_imagenet_upper}
  \end{subfigure}
  \caption{Accuracy under different maximum relative deadlines $D_u$ on CIFAR10 and ImageNet.}
  \label{fig:utility_upper}
\vspace{-0.3cm}
\end{figure}

\begin{figure}[!htb]
\vspace{-0.3cm}
\begin{subfigure}{1.0\linewidth}
  \centering
  \includegraphics[width=0.7\linewidth]{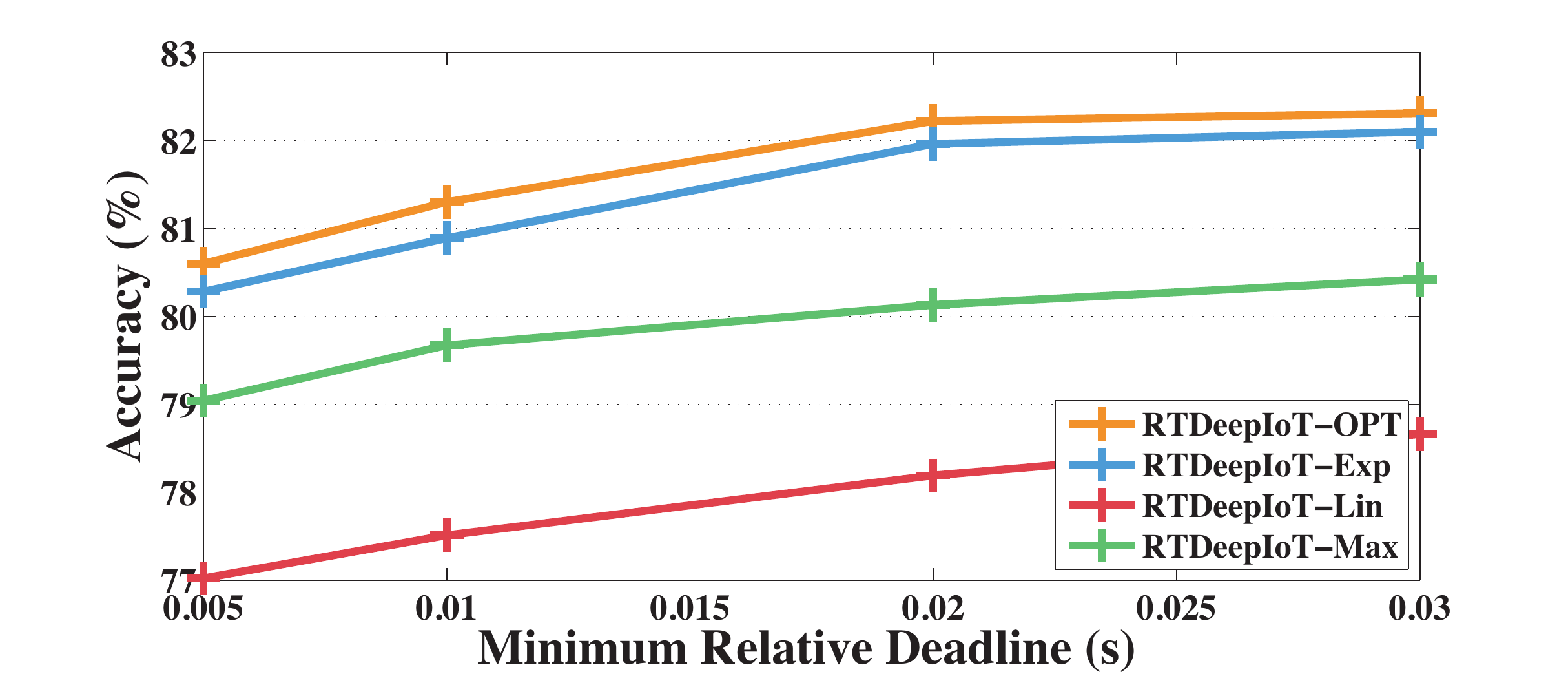}
  \caption{The accuracy with CIFAR10.}
  \label{fig:utility_cifar10_lower}
\end{subfigure}
\vspace{-0.2cm}
\begin{subfigure}{1.0\linewidth}
  \centering
  \includegraphics[width=0.7\linewidth]{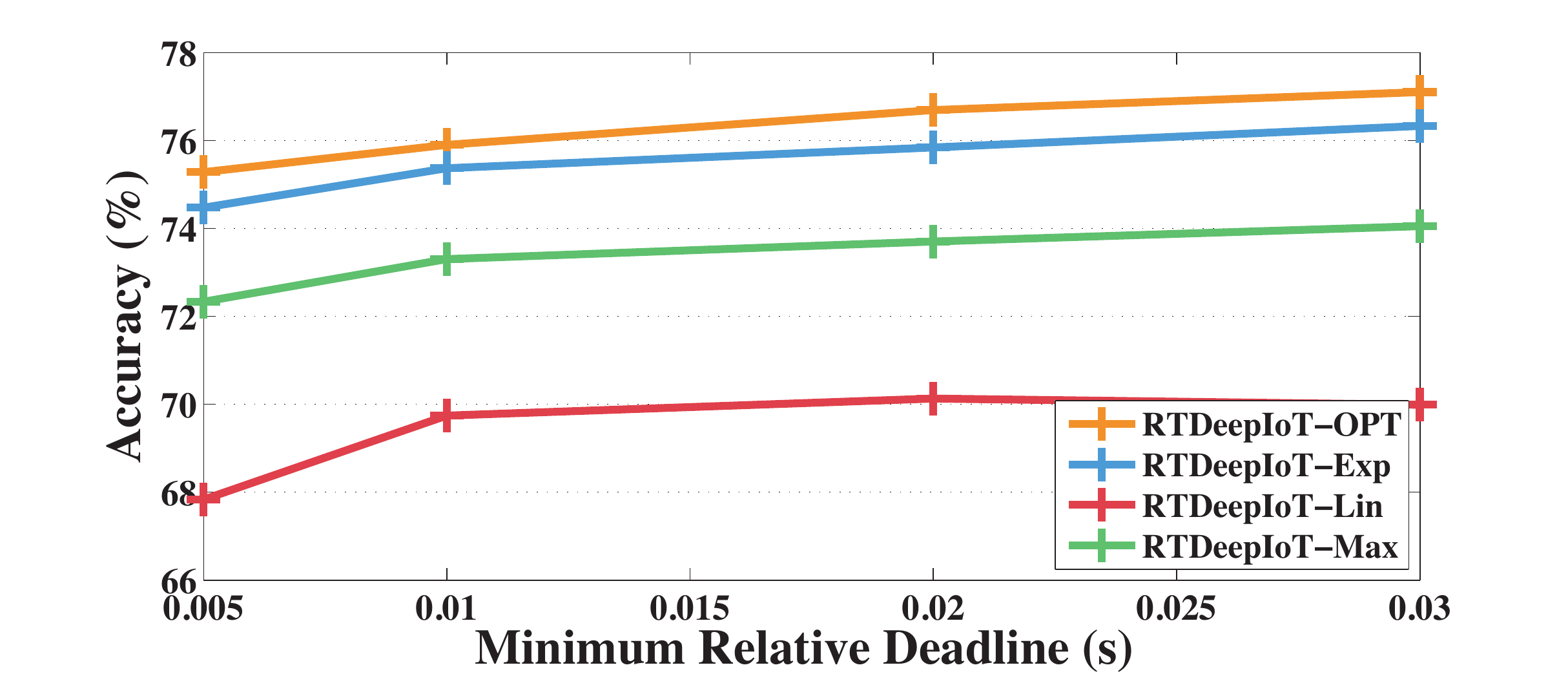}
  \caption{The accuracy with ImageNet.}
  \label{fig:utility_imagenet_lower}
  \end{subfigure}
  \caption{Accuracy under different minimum relative deadlines $D_l$ on CIFAR10 and ImageNet.}
  \label{fig:utility_lower}
\vspace{-0.2cm}
\end{figure}

\subsection{Stage Scheduling for Utility Maximization}
In this subsection, we evaluate the main objective of the RTDeepIoT framework; namely, maximizing cumulative utility of tasks. The following schedulers are compared:

\begin{enumerate}
\item RTDeepIoT: the scheduling algorithm proposed in Section~\ref{sec:DeepCrimson}. In this subsection, the scheduling algorithm uses the exponential increase heuristic for utility prediction, which can achieve the near optimal utility maximization scheduling as discussed in Section~\ref{sec:utility}.
Similarly, we define the quantization step of reward, $\Delta$, to be $0.1$. We will discuss the impact of quantization time step in Section~\ref{sec:hyperparameter}.
\item EDF: the traditional earliest deadline first algorithm. EDF only takes deadline constraints into consideration without considering maximizing overall utility. An attentive reader might remember that our GPU scheduling has limited preemption since individual stage execution cannot be interrupted. In Section~\ref{sec:nn-model}, we explained that accounting for the non-preemptible part (by subtracting one stage execution time from the end-to-end deadline) retains the order of deadlines of different tasks, which is why EDF remains a viable real-time baseline. 
\item LCF: the algorithm is called Least Confidence First, which picks the task with the least confidence. When two tasks have the same confidence, LCF will pick the one having an earlier deadline.
\item RR: a stage-level round-robin scheduling algorithm. The scheduler will select a stage to run among all existing tasks in a round-robin manner. RR implicitly takes the confidence into consideration by picking the task with the least executed stages.
\end{enumerate}

In all the following experiments, we evaluate these backend scheduling algorithms with two metrics, accuracy and deadline miss rate. If a service request cannot finish a single stage of the neural network before its deadline, that service request misses its deadline.
First, we evaluate these scheduling algorithms with an increasing number of concurrent clients $K$. Larger $K$ means more concurrent tasks will exist on average, leading to a intensity test for all backend schedulers. The average classification accuracy and deadline miss rates on CIFAR10 and ImageNet benchmarks are illustrated in Figure~\ref{fig:cifar10_task} and~\ref{fig:imagenet_task} respectively. RDeepIoT clearly outperforms all other schedulers with a large margin. By modeling deep learning tasks as imprecise computation models, RTDeepIoT is able to achieve both high classification accuracy (high service quality) and low deadline miss rate (high service responsiveness) at the same time under workloads with diverse intensity. 
The EDF scheduling algorithm has a suboptimal performance, because it fails to take utility into consideration. The EDF scheduler does not attempt to stop task execution prematurely (because it does not check if a good-enough answer has already been computed). It therefore imposes a higher load on the system, leading to low schedulability. EDF is known to do poorly under overload. By picking the task with the shortest deadline next, it always favors tasks that are more likely to miss their deadline. The LCF and RR schedulers achieve a lower deadline miss rate. However, by having to cut off tasks in a utility-insensitive manner (as they reach deadlines), they still result in inferior accuracy. 
\begin{figure}[!htb]
\begin{subfigure}{1.0\linewidth}
  \centering
  \includegraphics[width=0.7\linewidth]{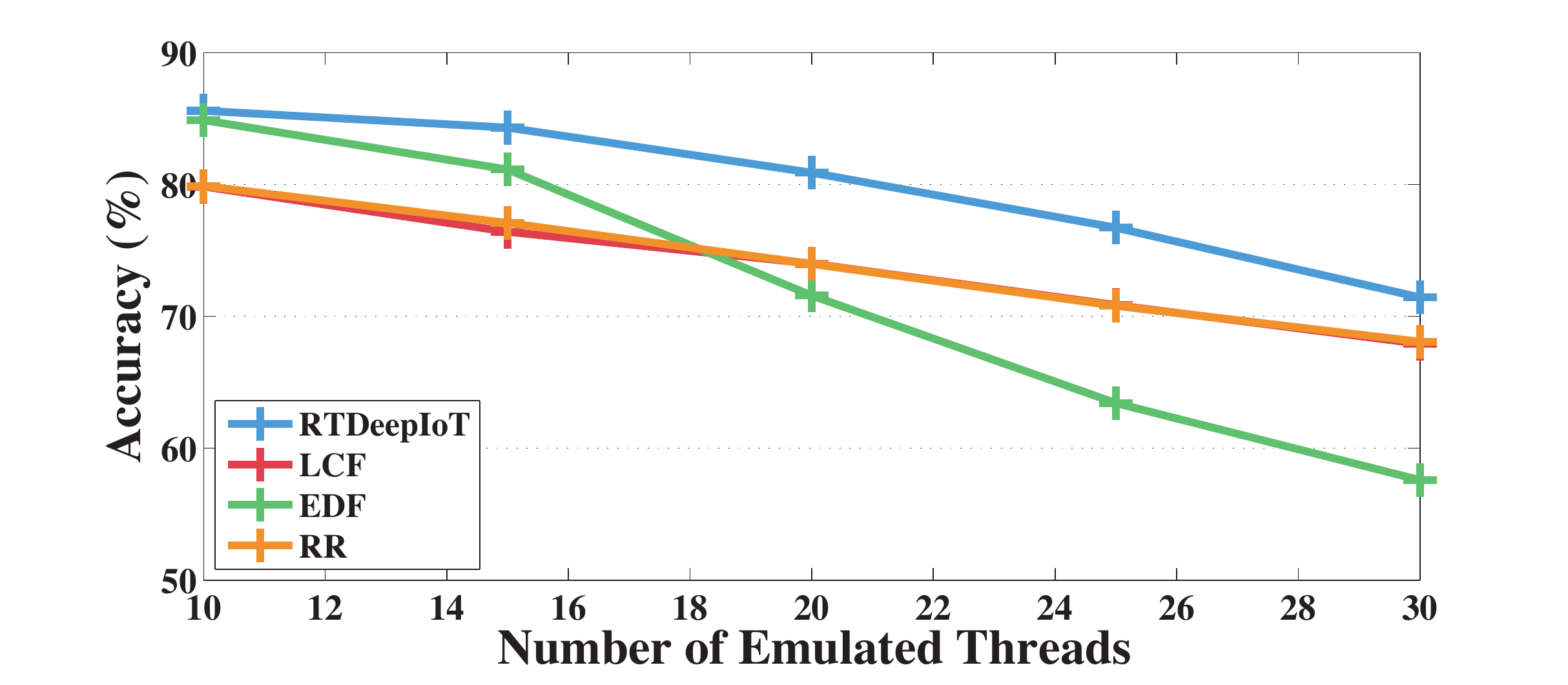}
  \caption{The accuracy with CIFAR10.}
  \label{fig:cifar10_task_acc}
\end{subfigure}
\vspace{-0.2cm}
\begin{subfigure}{1.0\linewidth}
  \centering
  \includegraphics[width=0.7\linewidth]{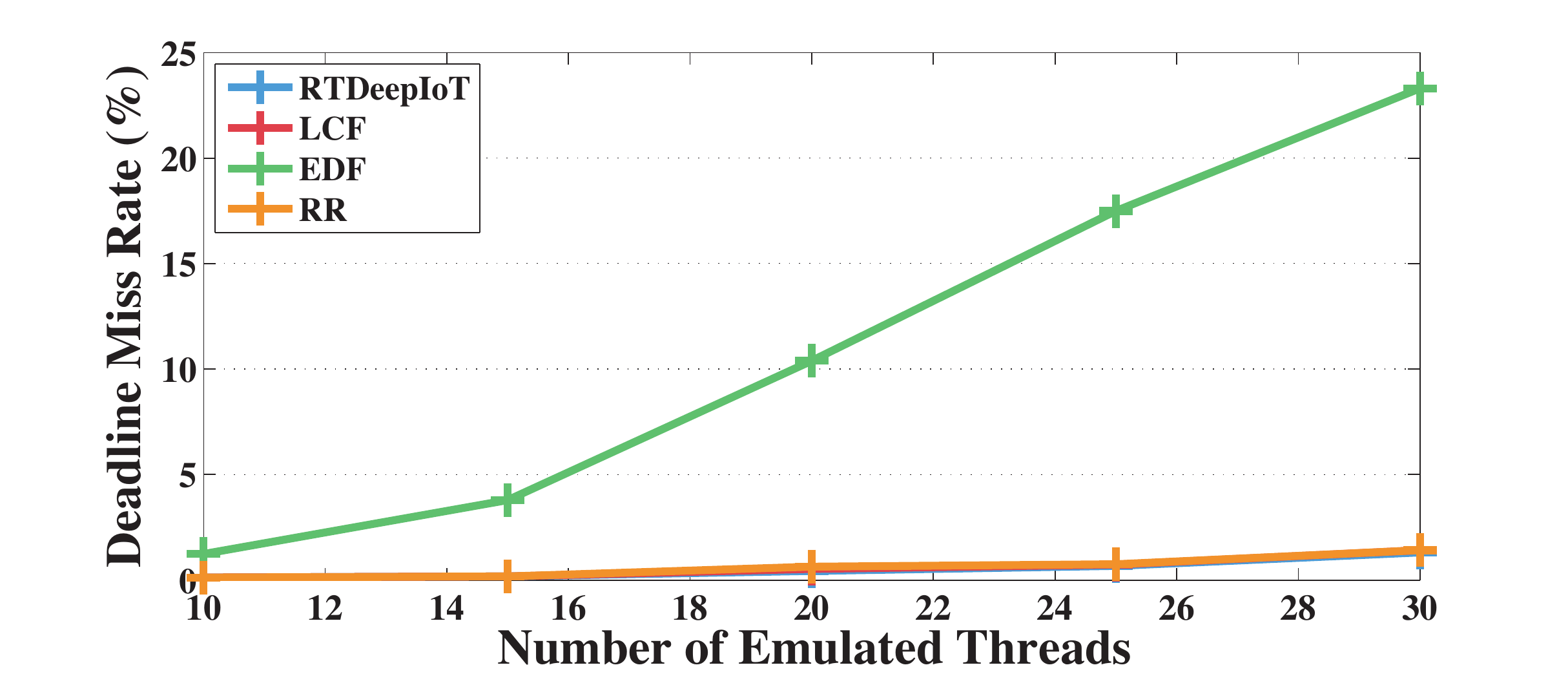}
  \caption{The deadline miss rate with CIFAR10.}
  \label{fig:cifar10_task_deadline}
  \end{subfigure}
  \caption{The performance under $K$ concurrent clients on CIFAR10.}
  \label{fig:cifar10_task}
\vspace{-0.3cm}
\end{figure}

\begin{figure}[!htb]
\vspace{-0.2cm}
\begin{subfigure}{1.0\linewidth}
  \centering
  \includegraphics[width=0.7\linewidth]{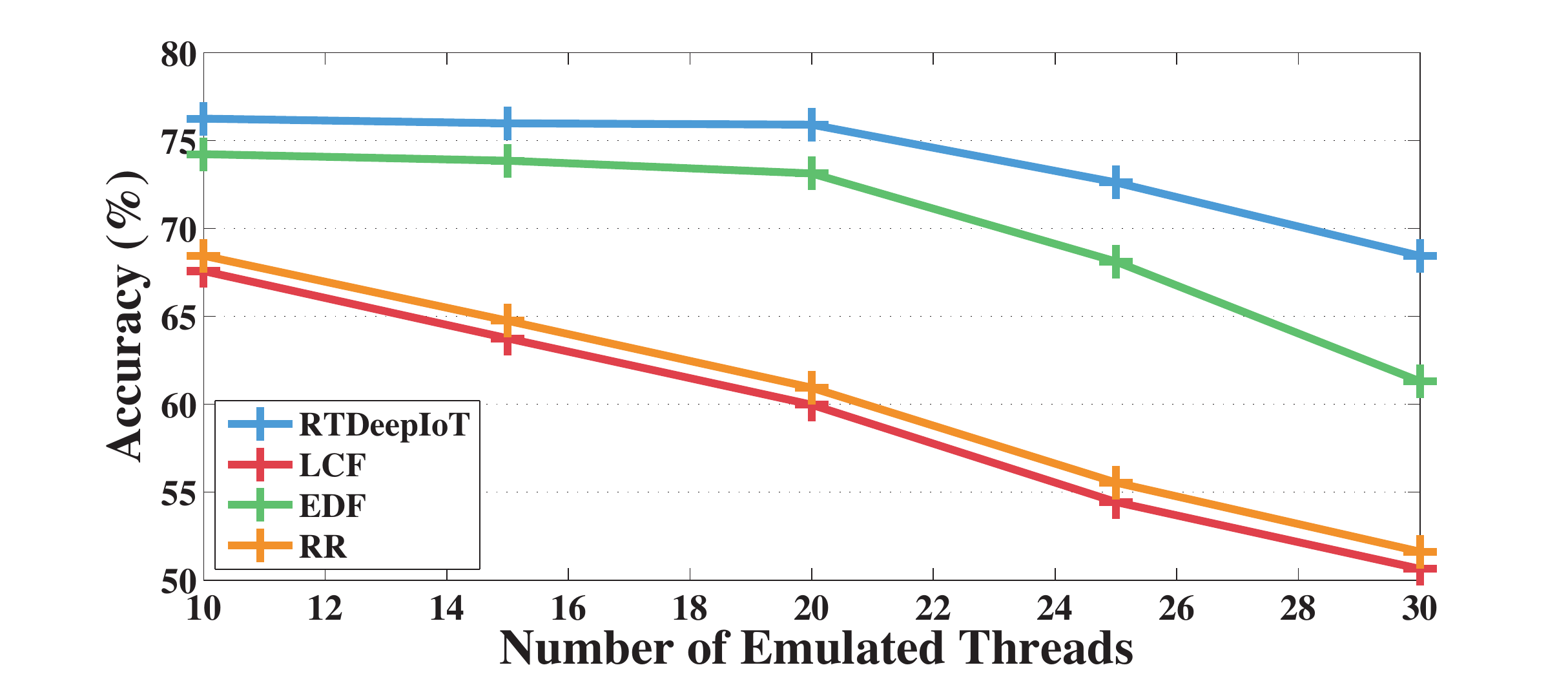}
  \caption{The accuracy with ImageNet.}
  \label{fig:imagenet_task_acc}
\end{subfigure}
\vspace{-0.2cm}
\begin{subfigure}{1.0\linewidth}
  \centering
  \includegraphics[width=0.7\linewidth]{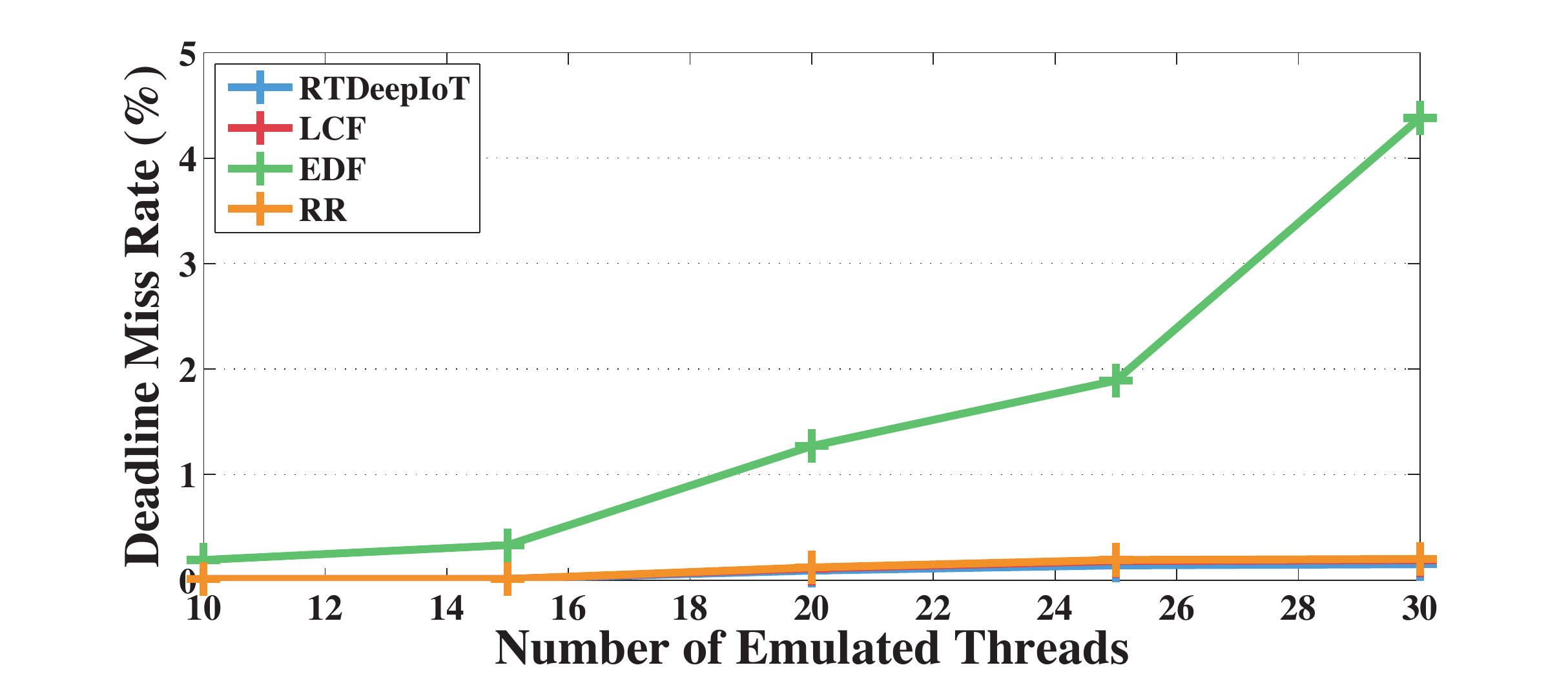}
  \caption{The deadline miss rate with ImageNet.}
  \label{fig:imagenet_task_deadline}
  \end{subfigure}
  \caption{The performance under $K$ concurrent clients on ImageNet.}
  \label{fig:imagenet_task}
\vspace{-0.3cm}
\end{figure}

\begin{figure}[!htb]
\vspace{-0.2cm}
\begin{subfigure}{1.0\linewidth}
  \centering
  \includegraphics[width=0.7\linewidth]{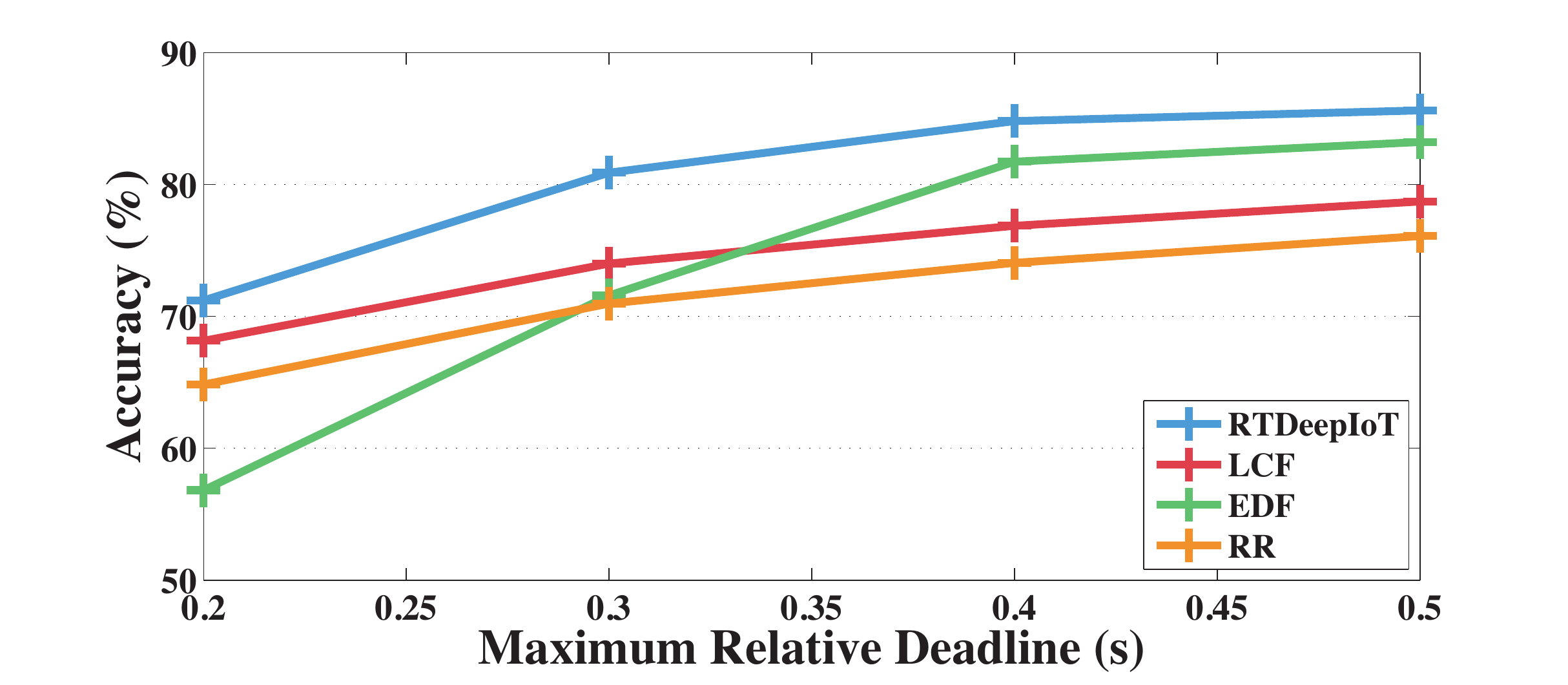}
  \caption{The accuracy with CIFAR10.}
  \label{fig:cifar10_upper_acc}
\end{subfigure}
\vspace{-0.2cm}
\begin{subfigure}{1.0\linewidth}
  \centering
  \includegraphics[width=0.7\linewidth]{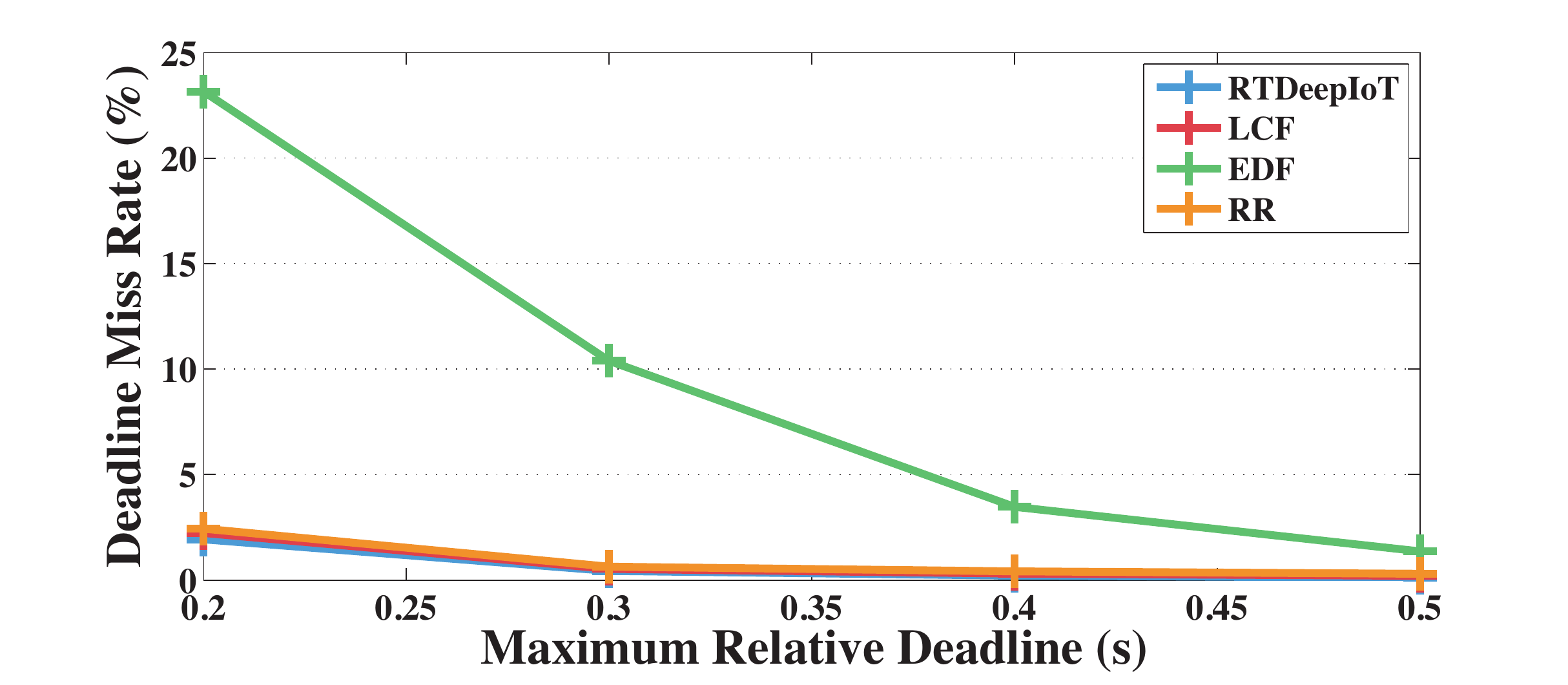}
  \caption{The miss rate with CIFAR10.}
  \label{fig:cifar10_upper_deadline}
  \end{subfigure}
  \caption{The performance under different maximum relative deadlines $D_u$ on CIFAR10.}
  \label{fig:cifar10_upper}
\vspace{-0.1cm}
\end{figure}

\begin{figure}[!htb]
\vspace{-0.3cm}
\begin{subfigure}{1.0\linewidth}
  \centering
  \includegraphics[width=0.7\linewidth]{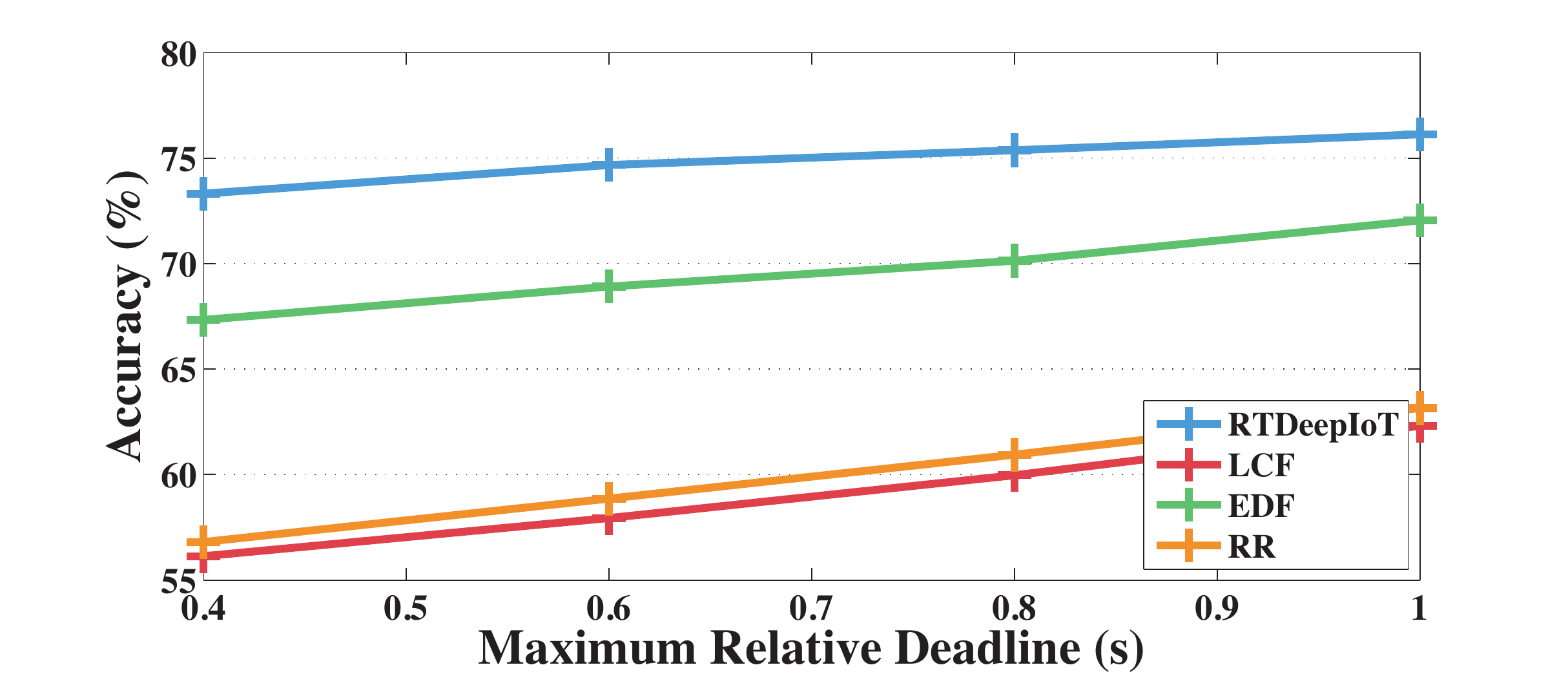}
  \caption{The accuracy with ImageNet.}
  \label{fig:imagenet_upper_acc}
\end{subfigure}
\vspace{-0.2cm}
\begin{subfigure}{1.0\linewidth}
  \centering
  \includegraphics[width=0.7\linewidth]{max_imagenet_UB_acc-eps-converted-to.pdf}
  \caption{The deadline miss rate with ImageNet.}
  \label{fig:imagenet_upper_deadline}
  \end{subfigure}
  \caption{The performance under different maximum relative deadlines $D_u$ on ImageNet.}
  \label{fig:imagenet_upper}
\vspace{-0.3cm}
\end{figure}

\begin{figure}[!htb]
\vspace{-0.3cm}
\begin{subfigure}{1.0\linewidth}
  \centering
  \includegraphics[width=0.7\linewidth]{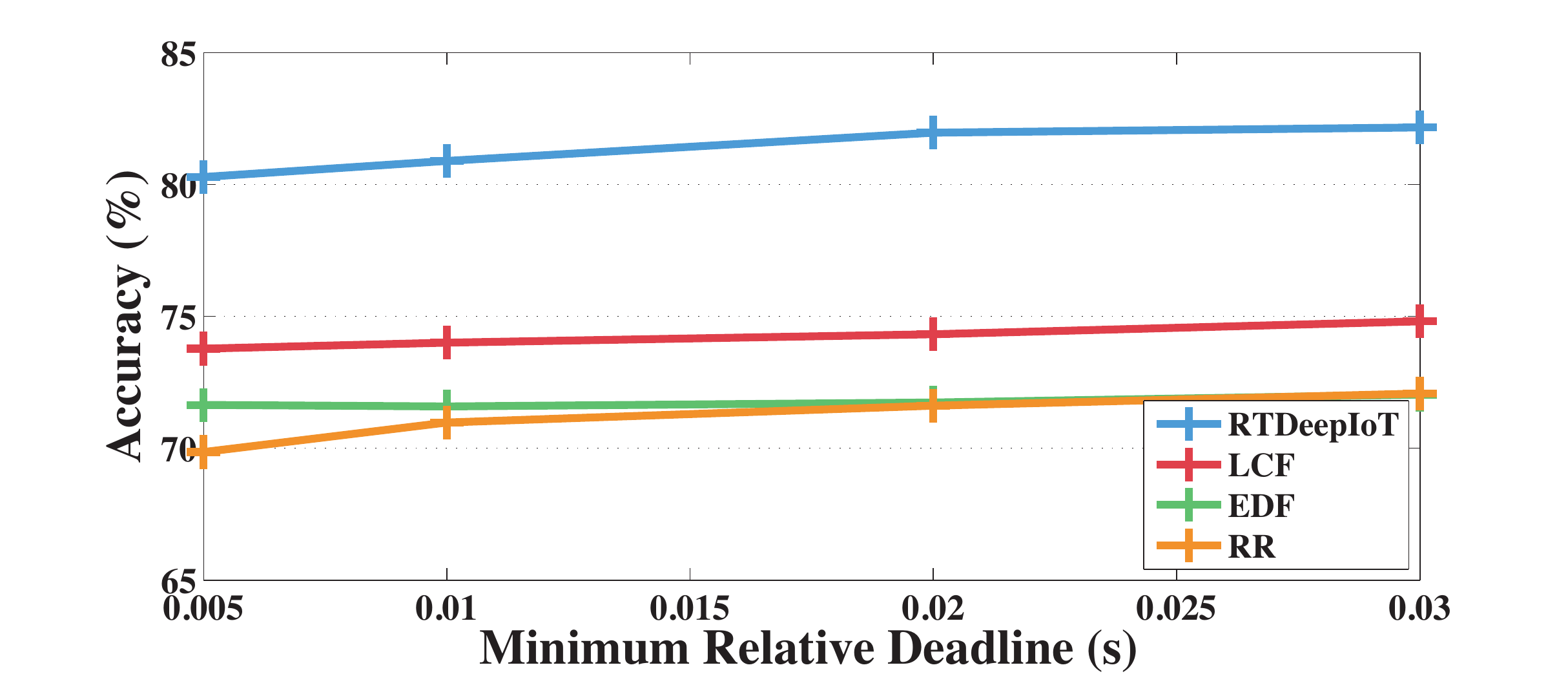}
  \caption{The accuracy with CIFAR10.}
  \label{fig:cifar10_lower_acc}
\end{subfigure}
\vspace{-0.2cm}
\begin{subfigure}{1.0\linewidth}
  \centering
  \includegraphics[width=0.7\linewidth]{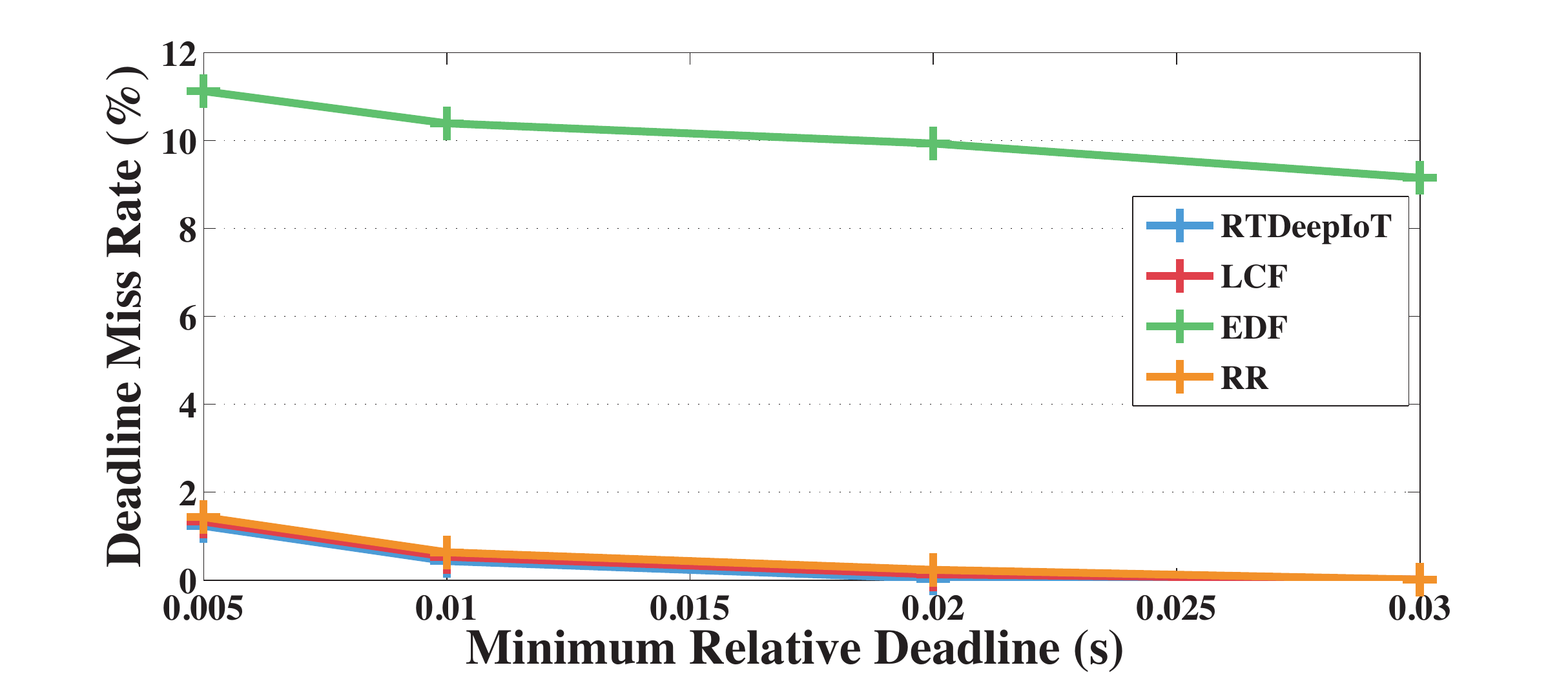}
  \caption{The deadline miss rate with CIFAR10.}
  \label{fig:cifar10_lower_deadline}
  \end{subfigure}
  \caption{The performance under different minimum relative deadlines $D_l$ on CIFAR10.}
  \label{fig:cifar10_lower}
\vspace{-0.3cm}
\end{figure}

\begin{figure}[!htb]
\vspace{-0.3cm}
\begin{subfigure}{1.0\linewidth}
  \centering
  \includegraphics[width=0.7\linewidth]{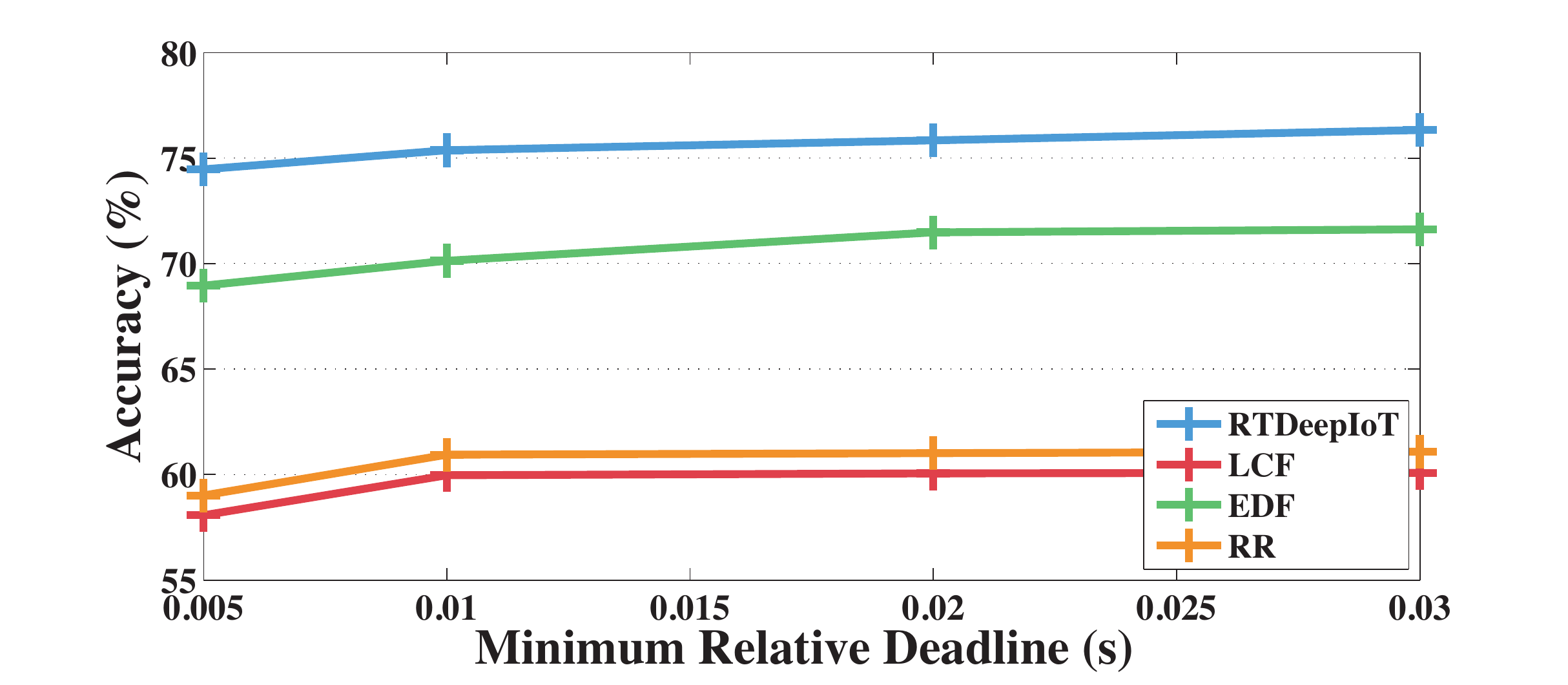}
  \caption{The accuracy with ImageNet.}
  \label{fig:imagenet_lower_acc}
\end{subfigure}
\vspace{-0.2cm}
\begin{subfigure}{1.0\linewidth}
  \centering
  \includegraphics[width=0.7\linewidth]{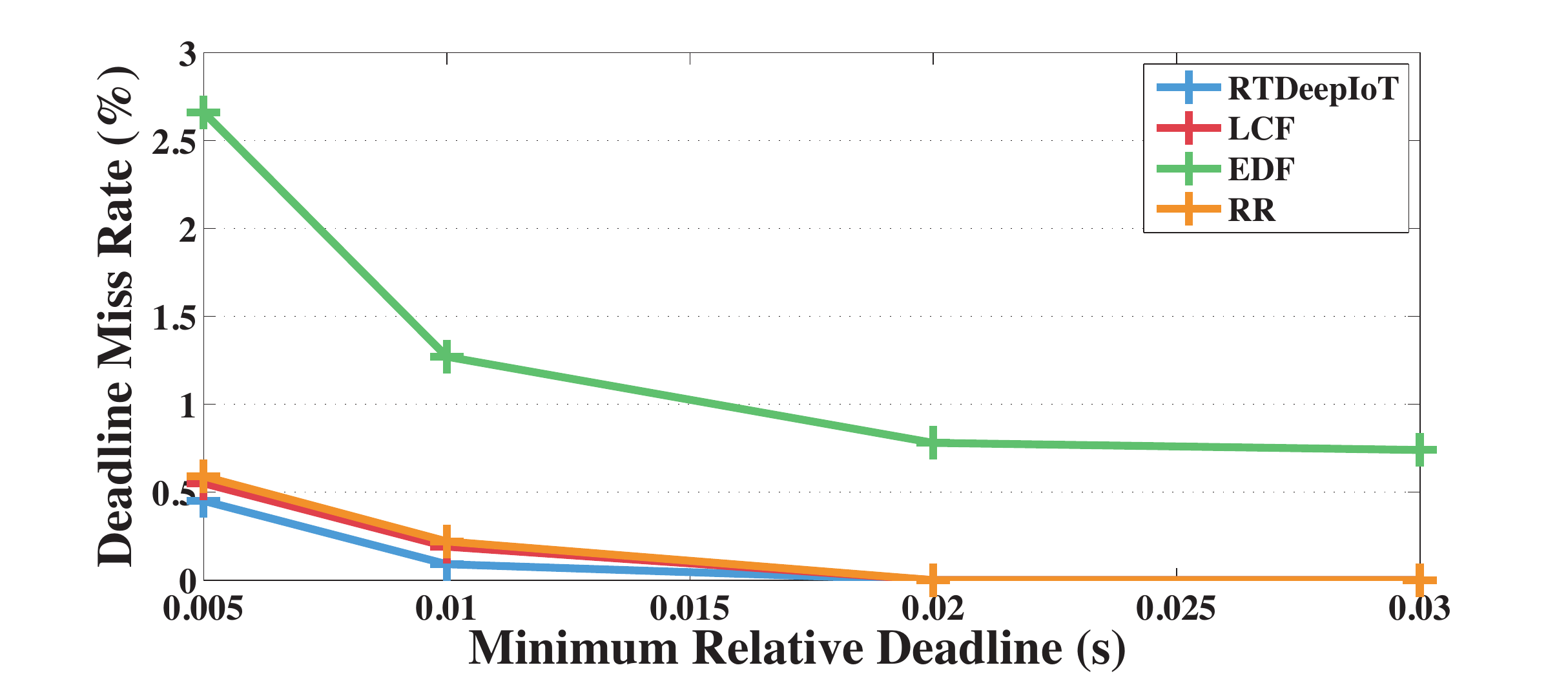}
  \caption{The deadline miss rate with ImageNet.}
  \label{fig:imagenet_lower_deadline}
  \end{subfigure}
  \caption{The performance under different minimum relative deadlines $D_l$ on ImageNet.}
  \label{fig:imagenet_lower}
\vspace{-0.1cm}
\end{figure}
All schedulers fails to provide good performance given an intensive workload. Note that, for our scheduler, deadline misses refer to images where no neural network stage got to execute. One can interpret them as images dropped by admission control (in that our utility-maximizing scheduler decided not to execute them).

\begin{figure}[!htb]
\begin{subfigure}{1.0\linewidth}
  \centering
  \includegraphics[width=0.7\linewidth]{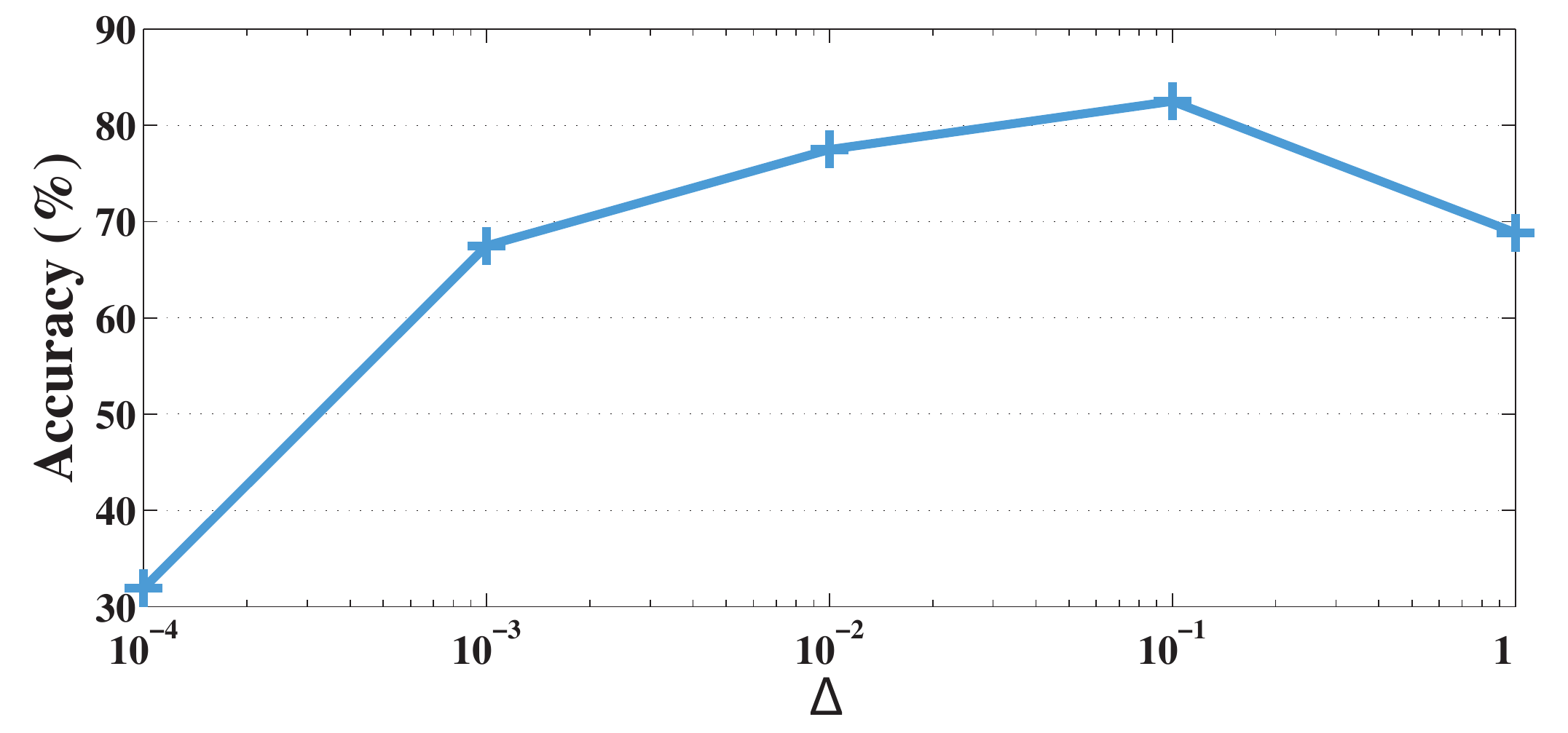}
  \caption{The accuracy on CIFAR10.}
  \label{fig:hyper_cifar10_acc}
\end{subfigure}
\vspace{-0.2cm}
\begin{subfigure}{1.0\linewidth}
  \centering
  \includegraphics[width=0.7\linewidth]{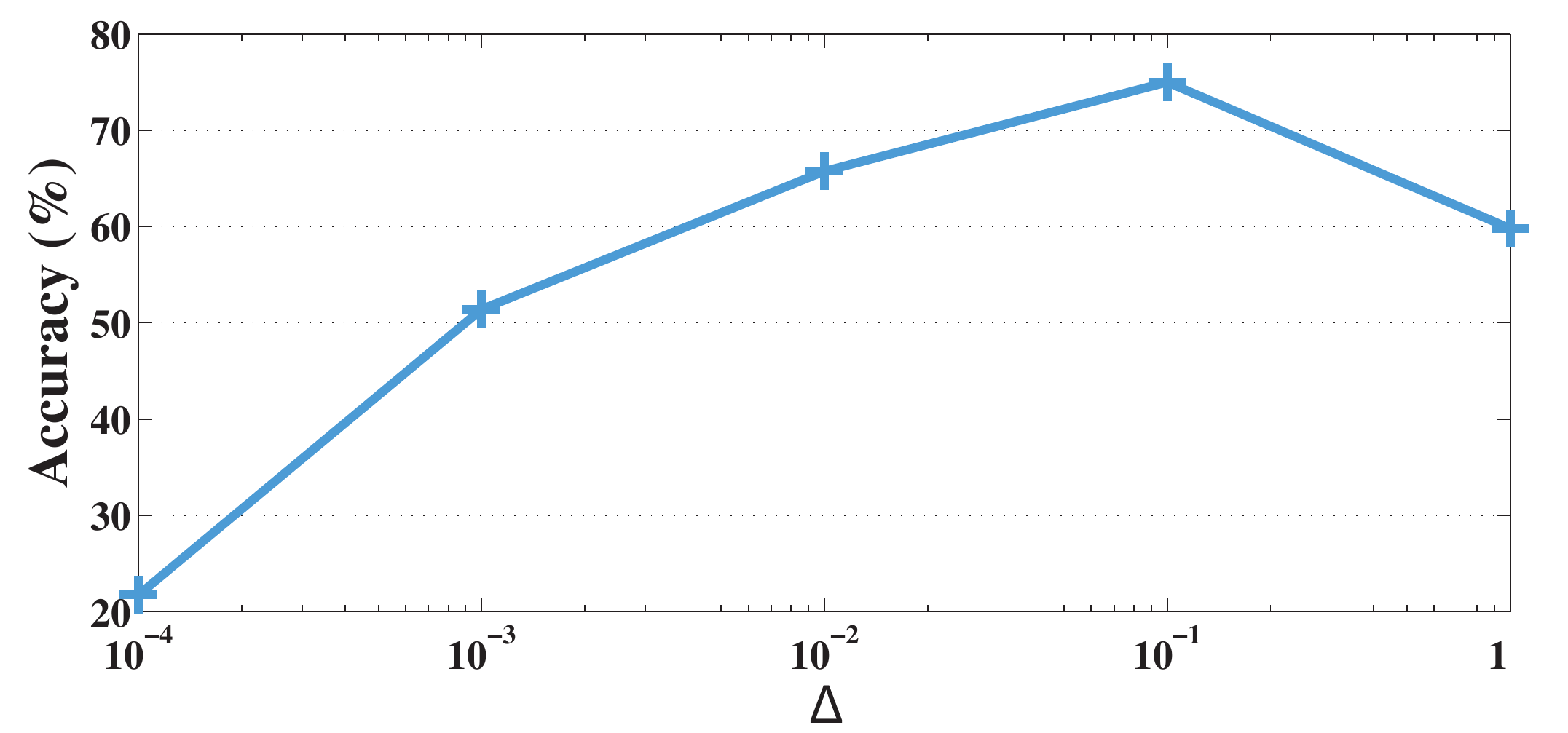}
  \caption{The accuracy on ImageNet.}
  \label{fig:hyper_imagenet_acc}
  \end{subfigure}
  \caption{The performance with a quantization reward step, $\Delta$.}
  \label{fig:hyper_acc}
\vspace{-0.3cm}
\end{figure}

\begin{figure}[!htb]
\vspace{-0.3cm}
\begin{subfigure}{1.0\linewidth}
  \centering
  \includegraphics[width=0.75\linewidth]{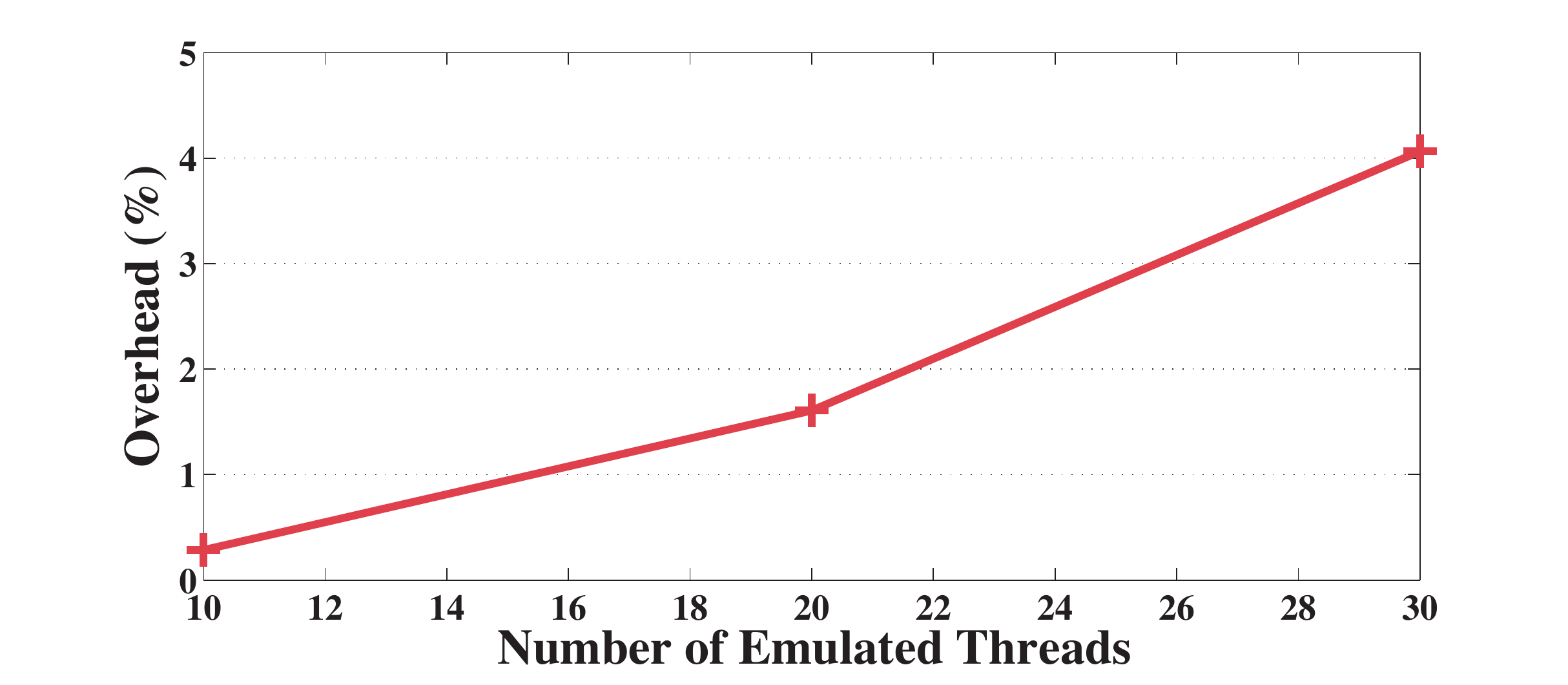}
  \caption{The overhead on CIFAR10.}
  \label{fig:overhead_cifar10_task}
\end{subfigure}
\vspace{-0.2cm}
\begin{subfigure}{1.0\linewidth}
  \centering
  \includegraphics[width=0.75\linewidth]{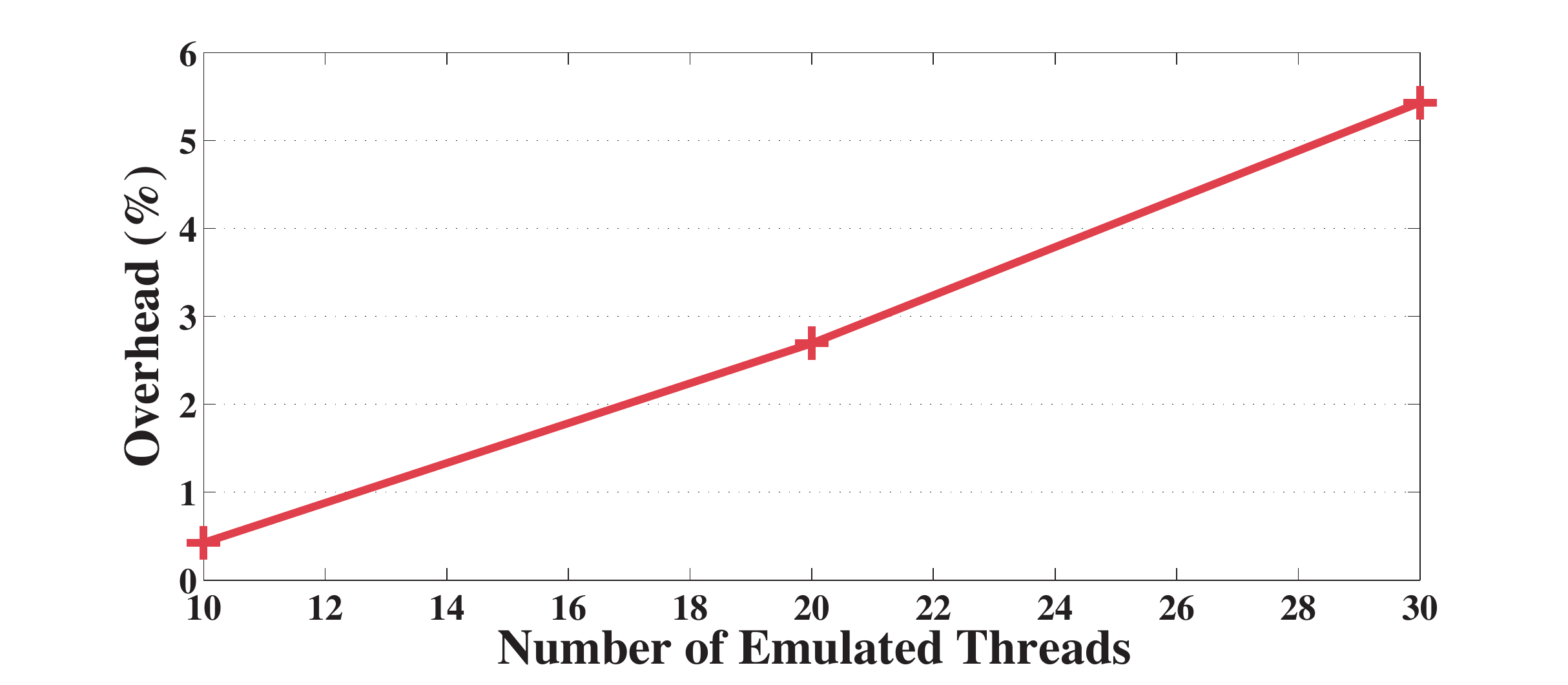}
  \caption{The overhead on ImageNet.}
  \label{fig:overhead_imagenet_task}
  \end{subfigure}
  \caption{The overhead with different service request threads $K$ on CIFAR10 and ImageNet.}
  \label{fig:overhead_task}
 \vspace{-0.1cm}
\end{figure}

Lastly, we evaluate backend schedulers by changing the maximum and minimum relative deadlines. Reducing the maximum and minimum relative deadlines causes tighter timing constraints for scheduling. Results are shown in Figure~\ref{fig:cifar10_upper} -~\ref{fig:imagenet_lower}. RTDeepIoT remains the best scheduler that achieves both high accuracy and low deadline miss rate. As shown in Figure~\ref{fig:cifar10_upper}, tighter timing constraints can have a larger impact on EDF that schedules entire tasks, not individual stages.

\subsection{Hyper-parameter Tuning}~\label{sec:hyperparameter}
In our scheduling algorithm, there exists a hyper-parameter $\Delta$ that controls the reward quantization step. 
In this subsection, we will investigate the impact of hyper-parameter $\Delta$ on final classification accuracy. We set all parameters concerning about workload patterns to be default values.

Evaluation results on CIFAR10 and ImageNet are illustrated Figure~\ref{fig:hyper_acc}. Hyper-parameter $\Delta$ has similar behaviors on both datasets. Theoretically speaking, smaller reward step $\Delta$ leads to a fine-grained scheduling policy with better cumulative utility (\ie high accuracy). However, smaller $\Delta$ also leads to higher time complexity of the scheduling algorithm, which can use up the time and resource for neural network execution, leading to worst cumulative utility (\ie low accuracy) and high deadline missing rate. Therefore, there exists a tradeoff in deciding hyper-parameter $\Delta$. From these experiments, we can see that $\Delta=0.1$ is a good choice for $\Delta$ in practice.

\subsection{System Overhead}
In this subsection, we evaluate the overhead of our scheduling framework. Since the framework contains a dynamic programming algorithm and a fast-updating greedy heuristics. We need to make sure these two modules do not impose significant overhead. 
The overhead is estimated as the averaged percentage of total time consumed by each service request except for the neural network execution time.
Similarly, we evaluate the overhead under diverse workload patterns. We adopt the same default value settings as mentioned in Section~\ref{sec:utility}.

As shown in Figure~\ref{fig:overhead_task}, we measure the overhead of scheduling by changing the number of service request threads $K$. The overhead of scheduling algorithm is between $0.5\%$ and $6\%$ in all these types of workloads. The reader is reminder that this overhead is for an un-optimized user-space implementation. Some reductions may be possible, by moving the scheduler to the kernel.  

\vspace{-0.2cm}
\section{Related Work}~\label{sec:Related}
\vspace{-0.7cm}

Our paper develops a real-time scheduler for a service motivated by machine intelligence needs of IoT applications.
A body of scheduling algorithms that comes close to ours are those that support approximate computing. 
A prime example of approximate computing in the real time research community is the literature on imprecise computations. The work trades off result quality versus computation time~\cite{Chung1987BS,Liu1991AlgorithmsFS,Chen2009ImpreciseCW,Amirijoo2006SpecificationAM,Feng1993AnEI}. Our work resembles imprecise computations in that we use intermediate results from a prematurely terminated real-time process. The work assumes processes to be monotone, and propose an indicator for the quality of the imprecise results. More recently~\cite{Chen2009ImpreciseCW}, imprecise computation models were proposed where the scheduler decides on the execution of an optional section of processes by taking deadlines and required QoS into consideration. Our work extends this concept to a novel application domain.

The QoS optimization and management have also been heavily addressed in real-time literature. Rajkumar et al. presented an analytical model for QoS management in systems with multiple constraints~\cite{rajkumar1997resource,rajkumar1998practical}. Lee et al. extended the QoS management analysis with the discrete and non-concave utility functions~\cite{lee1999quality}. Abdelzaher et al. proposed a real-time QoS negotiation model for maximizing system utility with guaranteed performance~\cite{atdelzater2000qos}. Curescu et al. presented a QoS optimization scheme for mobile networks~\cite{curescu2003time}. Koliver et al. designed a fuzzy-control approach for QoS adaptation~\cite{koliver2002specification}. In this paper, we adapt the utility optimization framework specifically to the execution of deep neural network tasks.
Moreover, our scheduler has been integrated with Tensor Flow; a library for deep learning systems~\cite{abadi2016tensorflow}. This makes it the first real-time scheduler implementing imprecise computations and utility maximization in the context of a mainstream deep learning software framework.

\vspace{-0.2cm}
\section{Conclusion}
\label{sec:Conclusion}
\vspace{-0.2cm}

This paper presented a novel scheduler, based on imprecise computations, suitable for edge services that simple devices (with sensing capabilities) offload their ``machine intelligence" to. We focused on deep learning as the state of the art enabler of machine intelligence. 
We show that the resulting schedules improve the average quality of results by allocating computing resources where they offer the best improvement in accuracy. The service is currently being extended to other deep learning libraries (besides machine vision) to offer rich support for deep intelligence as a (real-time) service.

\bibliographystyle{IEEEtran}
\bibliography{reference}

\balance

\end{document}